\documentclass[11pt, letterpaper]{article}
\usepackage{graphicx}
\usepackage[dvipsnames]{xcolor}

\usepackage[round]{natbib}

\usepackage{dirtytalk}

\usepackage{fullpage}

\usepackage[utf8]{inputenc} % allow utf-8 input
\usepackage[T1]{fontenc}    % use 8-bit T1 fonts
\usepackage{hyperref}       % hyperlinks
\usepackage{url}            % simple URL typesetting
\usepackage{booktabs}       % professional-quality tables
\usepackage{amsfonts}       % blackboard math symbols
\usepackage{nicefrac}       % compact symbols for 1/2, etc.
\usepackage{microtype}      % microtypography
\usepackage{xcolor}         % colors
\usepackage{enumerate}

\usepackage[colorinlistoftodos,bordercolor=orange,backgroundcolor=orange!20,linecolor=orange,textsize=scriptsize]{todonotes}

\usepackage{xspace}
\newcommand{\algname}[1]{{\sf \footnotesize #1}\xspace}

\input{preamble.tex}

\title{Streamlining in the Riemannian Realm:\\ Efficient Riemannian Optimization with Loopless Variance Reduction}
\date{}

\author{%
	Yury Demidovich \qquad Grigory Malinovsky \qquad Peter Richtárik \\
	\phantom{x}
	\\
	King Abdullah University of Science and Technology (KAUST) \\
	Thuwal, Saudi Arabia
}

\begin{document}
	
	\maketitle
	
	\begin{abstract}
		
		In this study, we investigate stochastic optimization on Riemannian manifolds, focusing on the crucial variance reduction mechanism used in both Euclidean and Riemannian settings. Riemannian variance-reduced methods usually involve a double-loop structure, computing a full gradient at the start of each loop. Determining the optimal inner loop length is challenging in practice, as it depends on strong convexity or smoothness constants, which are often unknown or hard to estimate. Motivated by Euclidean methods, we introduce the Riemannian Loopless \algname{SVRG} (\algname{R-LSVRG}) and \algname{PAGE} (\algname{R-PAGE}) methods. These methods replace the outer loop with probabilistic gradient computation triggered by a coin flip in each iteration, ensuring simpler proofs, efficient hyperparameter selection, and sharp convergence guarantees. Using R-PAGE as a framework for non-convex Riemannian optimization, we demonstrate its applicability to various important settings. For example, we derive Riemannian \algname{MARINA} (\algname{R-MARINA}) for distributed settings with communication compression, providing the best theoretical communication complexity guarantees for non-convex distributed optimization over Riemannian manifolds. Experimental results support our theoretical findings.
		
	\end{abstract}
	
	\section{Introduction}
	\label{submission}
	
	In this work we study finite-sum optimization problem: 
	\begin{equation} \label{eq:finite-sum}
		\min_{x\in\mathcal{X}\subset\mathcal{M}} \ f(x)\ \triangleq\ \frac{1}{n} \sum_{i=1}^{n} f_i(x),
	\end{equation}
	where $(\mathcal{M}, \mathfrak{g})$ represents a Riemannian manifold equipped with the Riemannian metric $\mathfrak{g}$, and $\mathcal{X} \subset \mathcal{M}$ is a geodesically convex set. Additionally, we assume that each function $f_i: \mathcal{M} \to \mathbb{R}$ is geodesically $L$-smooth.
	
	This formulation demonstrates its applicability across a wide range of practical scenarios, encompassing fundamental tasks such as principal component analysis (PCA) \citep{wold1987principal} and independent component analysis (ICA) \citep{lee1998independent}. Additionally, it extends its utility to address challenges like the completion and recovery of low-rank matrices and tensors \citep{tan2014riemannian, vandereycken2013low, mishra2014r3mc, kasai2016low}, dictionary learning \citep{cherian2016riemannian, sun2016complete}, optimization under orthogonality constraints \citep{edelman1998geometry, moakher2002means}, covariance estimation \citep{wiesel2012geodesic}, learning elliptical distributions \citep{sra2013geometric}, and Poincaré embeddings \citep{nickel2017poincare}. Furthermore, it proves effective in handling Gaussian mixture models \citep{hosseini2015matrix} and low-rank multivariate regression \citep{meyer2011linear}. This versatility makes it a valuable tool for tackling numerous problems in various settings.
	
	In addressing problems involving manifold constraints, a traditional approach involves alternating between optimization in the ambient Euclidean space and the process of "projecting" onto the manifold \citep{hauswirth2016projected}. The most popular method in this class is projected stochastic gradient descent \citep{luenberger1972gradient, calamai1987projected}. Furthermore, this concept is employed by other well-established methods. For example, two widely recognized techniques used to compute the leading eigenvector of symmetric matrices—power iteration \citep{bai1996some} and Oja’s algorithm \citep{oja1992principal}—are based on a projecting approach. However, these methods tend to suffer from high computational costs, when projecting onto certain manifolds (e.g., positive-definite matrices), which could be expensive in large-scale learning problems \cite{zhang2016riemannian, zhou2019faster}.
	
	An alternative option is to utilize Riemannian optimization, a method that directly interacts with the specific manifold under consideration \citep{da1998geodesic}. This approach enables Riemannian optimization to interpret the constrained optimization problem~(\ref{eq:finite-sum}) as an unconstrained problem on a manifold, eliminating the need for projections \citep{bonnabel2013stochastic, zhang2016first}. What's particularly significant is the conceptual perspective: by formulating the problem within a Riemannian framework, one can gain insights into the geometry of the problem \citep{zhang2018estimate}. This not only facilitates more precise mathematical analysis but also leads to the development of more efficient optimization algorithms.
	
	The expression of equation (\ref{eq:finite-sum}) in Euclidean form, where $\mathcal{M} = \mathbb{R}^d$ and $\mathfrak{g}$ represents the Euclidean inner product, has been a central focus of substantial algorithmic advancements in the fields of machine learning and optimization \citep{bertsekas2000gradient, goodfellow2016deep, sun2020optimization, ajalloeian2020convergence, demidovich2023guide}. This evolution traces back to the foundational work of Stochastic Gradient Method (\algname{SGD}) by \citet{robbins1951stochastic}. However, both batch and stochastic gradient methods grapple with considerable computational demands. When addressing finite sum problems with $n$ components, the full-gradient method requires $n$ derivatives at each step, while the stochastic method requires only one derivative \citep{bottou2018optimization}. Nevertheless, Stochastic Gradient Descent suffers from a slow sublinear rate \citep{gower2019sgd, khaled2020better, paquette2021sgd}. Tackling these challenges has spurred notable advancements in faster stochastic optimization within vector spaces, leveraging variance reduction techniques \citep{schmidt2017minimizing, johnson2013accelerating, defazio2014saga, konevcny2017semi}.
	
	In conjunction with numerous recent studies \citep{song2020variance,gower2020variance, dragomir2021fast}, these algorithms showcase accelerated convergence compared to the original gradient descent algorithms across various scenarios, including strongly convex \citep{gorbunov2020unified}, general convex \citep{khaled2023unified}, and non-convex settings \citep{reddi2016stochastic, allen2016variance, fang2018spider}. In order to simplify complicated structures of variance-reduced methods, loopless versions were proposed initially in strongly convex settings \citep{hofmann2015variance, kovalev2020don}. Later, this approach was adopted in non-convex and general convex settings \citep{li2021page, horvath2022adaptivity, khaled2023unified}. The loopless structure allows us to obtain practical parameters and make proofs more elegant.
	
	In the context of distributed learning, where each individual function in equation \eqref{eq:finite-sum} is stored on a separate device, the communication compression approach is frequently utilized to ease communication load \citep{alistarh2017qsgd}. The initial concept involved employing quantization or sparsification of gradients, sending compressed gradients to the server for aggregation, and subsequently performing a step \citep{wangni2018gradient}. However, owing to the high variance or errors linked with compression operators, a direct application does not consistently ensure improved convergence \citep{alistarh2018convergence, shi2019convergence}. To tackle this issue, akin to the \algname{SGD} scenario, various mechanisms for variance reduction \citep{mishchenko2019distributed, horvath2023stochastic} and error compensation in compression \citep{stich2018sparsified, stich2019error, richtarik2021ef21} have been proposed. The current state-of-the-art method in non-convex scenario is \algname{MARINA} \citep{gorbunov2021marina}, which allows to obtain optimal $\mathcal{O}\left(\frac{1+\omega / \sqrt{n}}{\varepsilon^2}\right)$ communication complexity. 
	
	In recent times, there has been a growing interest in exploring the Riemannian counterparts of batch and stochastic optimization algorithms. The pioneering work by \citet{zhang2016first} marked the first comprehensive analysis of the global complexity of batch and stochastic gradient methods for geodesically convex functions. Subsequent research by \citet{zhang2016riemannian, kasai2016riemannian, sato2019riemannian, han2021improved} focused on improving the convergence rate for finite-sum problems through the application of variance reduction techniques. Additionally, there has been an analysis of the near-optimal variance reduction method for non-convex Riemannian optimization problems, known as \algname{R-SPIDER} \citep{zhang2018r, zhou2019faster}. While this method boasts strong theoretical guarantees, its practical implementation poses challenges due to its double-loop structure. Determining practical parameters is particularly difficult, as the length of the inner loop often depends on condition number or smoothness constant and the desired level of accuracy, denoted by $\varepsilon$. 
	
	\section{Contributions} 
	Below, we outline the primary contributions of this paper.
	\begin{itemize}
		\item We introduce \algname{R-LSVRG}, a Riemannian Loopless Stochastic Variance Reduced Gradient Descent method inspired by the Euclidean \algname{L-SVRG} method \citep{kovalev2020don}. In its loopless form, we eliminate the inner loop and replace it with a biased coin-flip mechanism executed at each step of the method. This coin-flip determines when to compute the gradient by making a pass over the data. Alternatively, this approach can be interpreted as incorporating an inner loop of random length. The resulting method is easier to articulate, understand, and analyze. We demonstrate that \algname{R-LSVRG} exhibits the same rapid theoretical convergence rate of $\mathcal{O}\left(\left(n + \frac{L^2\zeta}{\mu^2}\right)\log\frac{1}{\varepsilon}\right)$ as its looped counterpart \citep{zhang2016riemannian}. Moreover, our analysis allows the expected length of the inner loop to be  $\mathcal{O}(n)$, independent of the strong convexity constant $\mu$ and the smoothness constant $L$, making the method more practically applicable.
		
		\item We present \algname{R-PAGE}, a Riemannian adaptation of the Probabilistic Gradient Estimator designed for non-convex optimization, drawing inspiration from the research conducted by \citet{li2021page}. We analyze \algname{R-PAGE} for optimizing geodesically smooth stochastic non-convex functions. Our analysis shows that this method achieves the best-known rates of $\mathcal{O}\left(n+\frac{n^{1 / 2}}{\varepsilon^2}\right)$ in the finite sum setting and $\mathcal{O}\left(\frac{1}{\varepsilon^3}\right)$ in the online setting. Moreover, these guarantees align with the lower bound established by \citep{fang2018spider} in the Euclidean case. Similar to the \algname{R-LSVRG} method, the \algname{R-PAGE} algorithm also employs a coin-flip approach, making the length of the inner loop random. Our analysis allows us to choose the expected length to be $\mathcal{O}(n)$, independent of the smoothness constant $L$ and the accuracy $\varepsilon$, rendering the method practical.
		
		\item Employing \algname{R-PAGE} as a foundation for non-convex Riemannian optimization, we showcase its adaptability across diverse and significant contexts. As an illustration, we formulate Riemannian \algname{MARINA} (\algname{R-MARINA}), specifically tailored for distributed scenarios incorporating communication compression and variance reduction. This development not only extends the utility of \algname{R-PAGE} but also establishes, to the best of our knowledge, the best theoretical communication complexity assurances for non-convex distributed optimization over Riemannian manifolds, aligning with the lower bounds in the Euclidean case \citep{gorbunov2021marina}.

	\end{itemize}
	
	\section{Preliminaries}

	A \emph{Riemannian manifold} $(\mathcal{M}, \mathfrak{g})$ is a real smooth manifold $\mathcal{M}$ equipped with a Riemannian metric $\mathfrak{g}$. This metric induces an inner product structure in each tangent space $T_x\mathcal{M}$ associated with every $x\in\mathcal{M}$. The inner product of vectors $u$ and $v$ in $T_x\mathcal{M}$ is denoted as $\langle u, v \rangle \triangleq \mathfrak{g}_x(u,v)$, and the norm of a vector $u\in T_x\mathcal{M}$ is defined as $\|u\| \triangleq \sqrt{\mathfrak{g}_x(u,u)}$. The angle between vectors $u$ and $v$ is given by $\arccos\frac{\langle u, v \rangle}{\|u\|\|v\|}$.
	
	A geodesic is a curve parameterized by constant speed $\gamma: [0,1]\to\mathcal{M}$ that locally minimizes distance. The exponential map $\Expmap_x:T_x\mathcal{M}\to\mathcal{M}$ maps a vector $v$ in $T_x\mathcal{M}$ to a point $y$ on $\mathcal{M}$ such that there exists a geodesic $\gamma$ with $\gamma(0) = x$, $\gamma(1) = y$, and $\dot{\gamma}(0) \triangleq \frac{d}{dt}\gamma(0) = v$.
	
	If there is a unique geodesic between any two points in $\mathcal{X}\subset\mathcal{M}$, the exponential map has an inverse $\Expmap_x^{-1}:\mathcal{X}\to T_x\mathcal{M}$. The geodesic is the uniquely shortest path with $\|\Expmap_x^{-1}(y)\| = \|\Expmap_y^{-1}(x)\|$, defining the geodesic distance between $x$ and $y\in\mathcal{X}$.
	
	Parallel transport, denoted as $\Gamma_x^y: T_x\mathcal{M}\to T_y\mathcal{M}$, is a mapping that transports a vector $v\in T_x\mathcal{M}$ to $\Gamma_x^y v\in T_y\mathcal{M}$. This process retains both the norm and, in a figurative sense, the "direction," akin to translation in $\mathbb{R}^d$. Notably, a tangent vector of a geodesic $\gamma$ maintains its tangential orientation when parallel transported along $\gamma$. Moreover, parallel transport preserves inner products.
	
	In our work, we will employ various crucial definitions and standard assumptions for theoretical analysis.
	
	\begin{definition}[Riemannian gradient]
		The Riemannian gradient, denoted as $\nabla f(x)$, represents a vector in the tangent space $T_x(\mathcal{M})$ such that $\frac{d(f(\gamma(t)))}{dt}\vert_{t=0}=\langle v, \nabla f(x)\rangle_x,$ holds true for any $v\in T_x\mathcal{M}.$
	\end{definition}

	\begin{assumption}[Geodesic convexity] 
		A function $f:\mathcal{X}\to\mathbb{R}$ is considered geodesically convex if, for any $x$ and $y$ in $\mathcal{X}$, and for any geodesic $\gamma$ connecting $x$ to $y$ such that $\gamma(0)=x$ and $\gamma(1)=y$, the inequality
		\[ f(\gamma(t)) \le (1-t)f(x) + tf(y) \]
		holds for all $t$ in the interval $[0,1].$
	\end{assumption}
	%It can be shown that an equivalent definition is that for any $x,y\in\mathcal{X}$,
	%\[ f(y) \ge f(x) + \langle g_x, \Expmap_x^{-1}(y) \rangle_x, \]
	%where $g_x$ is a subgradient of $f$ at $x$, or the gradient if $f$ is differentiable, and $\langle\cdot,\cdot\rangle_x$ denotes the inner product in the tangent space of $x$ induced by the Riemannian metric. In the rest of the paper we will omit the index of tangent space when it is clear from the context.
	
	Demonstrably, an equivalent definition asserts that, for any $x$ and $y$ in $\mathcal{X}$,
	\[ f(y) \ge f(x) + \langle g_x, \Expmap_x^{-1}(y) \rangle_x, \]
	where $g_x$ serves as a subgradient of $f$ at $x$, or the gradient in case $f$ is differentiable. Here, $\langle\cdot,\cdot\rangle_x$ signifies the inner product within the tangent space at $x$ induced by the Riemannian metric. 
	\begin{assumption}[Geodesic strong convexity]
		\label{ass:g-strong-convexity}
		A function $f:\mathcal{X}\to\mathbb{R}$ is considered geodesically $\mu$-strongly convex ($\mu$-strongly g-convex) if, for any $x$ and $y$ in $\mathcal{X}$ and a subgradient $g_x$, the inequality holds:
		\[ f(y) \ge f(x) + \langle g_x, \Expmap_x^{-1}(y) \rangle_x + \frac{\mu}{2}d^2(x,y).\]

		%	A function $f:\mathcal{X}\to\mathbb{R}$ is said to be geodesically $\mu$-strongly convex ($\mu$-strongly g-convex) if for any $x,y\in\mathcal{X}$ and subgradient $g_x,$ it holds that
		%	\[ f(y) \ge f(x) + \langle g_x, \Expmap_x^{-1}(y) \rangle_x + \frac{\mu}{2}d^2(x,y).\]
	\end{assumption}
	\begin{assumption}[Geodesic smoothness]
		\label{ass:l-g-smoothness}
		%	A differentiable function $f:\mathcal{X}\to\mathbb{R}$ is said to be geodesically $L$-smooth ($L$-g-smooth), if its gradient is geodesically $L$-Lipschitz ($L$-g-Lipschitz), i.e. for any $x,y\in\mathcal{X},$
		%	\begin{eqnarray}\label{eq:l-g-smoothness}
			%	\|g_x - \Gamma_y^x g_y\| & \le & L\|\Expmap_x^{-1}(y)\|,
			%	\end{eqnarray}
		%	where $\Gamma_y^x$ is the parallel transport from $y$ to $x.$ In this case we have that
		%	\begin{equation*}
			%	f(y) \leq f(x) + \langle g_x, \Expmap^{-1}_x(y) \rangle + \frac{L}{2}\norm{\Expmap^{-1}_x(y)}^2.
			%	\end{equation*}
		
		A differentiable function $f:\mathcal{X}\to\mathbb{R}$ is considered geodesically $L$-smooth ($L$-g-smooth) if its gradient is geodesically $L$-Lipschitz ($L$-g-Lipschitz). This is expressed as, for any $x$ and $y$ in $\mathcal{X}$,
		\begin{eqnarray*}\label{eq:l-g-smoothness}
			\|g_x - \Gamma_y^x g_y\|  \leq L\|\Expmap_x^{-1}(y)\|,
		\end{eqnarray*}
		where $\Gamma_y^x$ represents the parallel transport from $y$ to $x$. In such cases, the following inequality holds:
		\begin{equation*}
			f(y) \leq f(x) + \langle g_x, \Expmap^{-1}_x(y) \rangle + \frac{L}{2}\norm{\Expmap^{-1}_x(y)}^2.
		\end{equation*}
	\end{assumption}
	
	Throughout the remainder of the paper, we will omit the index of the tangent space when its context is apparent.
	
	\begin{assumption}[Polyak-\L ojasiewicz condition]
		\label{ass:pl-condition}
		We assert that $f:\mathcal{M}\to\mathbb{R}$ satisfies the Polyak-\L ojasiewicz condition (P\L-condition) if $f^{*}$ uniformly bounds $f(x)$ from below for all $x\in\mathcal{M}$, and there exists $\mu>0$ such that
		\begin{equation*}
			2\mu\left(f(x) - f^*\right)\leq \norm{f(x)}^2, \quad \forall; x\in\mathcal{X}.
		\end{equation*}
		
		\begin{assumption}[Uniform lower bound]
			\label{ass:lower}
			There exists $f^* \in$ $\mathbb{R}$ such that $f(x) \geq f^*$ for all $x \in \mathbb{R}^d$.
		\end{assumption}
		
	\end{assumption}
	We revisit an essential trigonometric distance bound, which holds significance for our analytical considerations. 
	\begin{lemma}[\citep{zhang2016first} {Lemma 6}]\label{lemma:distance_bound}
		If $a$, $b$, and $c$ represent the sides (i.e., side lengths) of a geodesic triangle in an Alexandrov space with curvature lower bounded by $\kappa_{\min}$, and $A$ denotes the angle between sides $b$ and $c$, then the following distance bound holds:
		\begin{equation} \label{eq:distance_bound}
			a^2 \leq \frac{\sqrt{|\kappa_{\min}|}c}{\tanh(\sqrt{|\kappa_{\min}|}c)}b^2 + c^2 - 2bc\cos(A).
		\end{equation}
	\end{lemma}
	In the subsequent sections, we adopt the notation $\zeta(\kappa_{\min},c)\eqdef \frac{\sqrt{|\kappa_{\min}|}c}{\tanh(\sqrt{|\kappa_{\min}|}c)}$ for the curvature-dependent quantity defined in inequality~\eqref{eq:distance_bound}. Leveraging Lemma~\ref{lemma:distance_bound}, we can readily establish the following corollary. This corollary unveils a significant relationship between two consecutive updates within an iterative optimization algorithm on a Riemannian manifold with curvature bounded from below.
	\begin{corollary} \label{th:distance_corollary}
		For any Riemannian manifold $\mathcal{M}$ where the sectional curvature is lower bounded by $\kappa_{\min}$ and for any point $x$, $x_s\in\mathcal{M}$, the update $x_{s+1} = \Expmap_{x_s}(-\eta_s g_{x_s})$ adheres to the inequality:
		\begin{align*} \label{eq:distance_corollary}
			\notag	\langle -g_{x_s}, \Expmap_{x_s}^{-1}(x)\rangle& \le \frac{1}{2\eta_s} \left(d^2(x_{s},x) - d^2(x_{s+1},x)\right)
			+ \frac{\zeta(\kappa_{\min},d(x_s,x))\eta_s}{2} \|g_{x_s}\|^2.
		\end{align*}
	\end{corollary}
	In our analysis, we employ an additional set of assumptions essential for handling the geometry of Riemannian manifolds. It is worth noting that the majority of practical manifold optimization problems can meet these assumptions.
	\begin{assumption} We assume that
		\label{ass:geom}	
		\begin{enumerate}[(a)]
			\item	$f$ attains its optimum at $x^*\in\mathcal{X}$;
			\item $\mathcal{X}$ is compact, and the diameter of $\mathcal{X}$ is bounded by $D$, that is, $\max_{x,y\in\mathcal{X}} d(x,y) \le D$;
			\item the sectional curvature in $\mathcal{X}$ is \emph{upper} bounded by $\kappa_{\max}$, and within $\mathcal{X}$ the exponential map is invertible;
			\item the sectional curvature in $\mathcal{X}$ is \emph{lower} bounded by $\kappa_{\min}$. 
		\end{enumerate}
	\end{assumption}
	We introduce a fundamental geometric constant with the objective of encapsulating and characterizing the impact and importance associated with the curvature of the manifold.
	\begin{definition}[Curvature-driven manifold term]
		\label{eq:zeta-sigma}
		The term "curvature-driven manifold" refers to the constant described below:
		\begin{equation*} 
			\zeta= \left\{\begin{array}{ll}
				\frac{\sqrt{|\kappa_{\min}|}D}{\tanh(\sqrt{|\kappa_{\min}|}D)}, & \text{if } \kappa_{\min} < 0, \\
				1, & \text{if } \kappa_{\min} \ge 0.
			\end{array}\right.
		\end{equation*}

	\end{definition}

	%	\begin{assumption}[$L$-g-smoothness]\label{ass:l-g-smoothness}
		%	Let function $f$ be differentiable and $L$-g-smooth on $\mathcal{M}.$
		%\end{assumption}
		%\begin{assumption}[Uniform lower bound]\label{ass:lower-bound}
		%	There exists $f^{*}\in\mathbb{R}$ such that $f(x)\geq f^{*}$ for all $x\in\mathcal{M}.$ 
		%\end{assumption}
		%\begin{assumption}[Strong $g$-convexity]\label{ass:g-strong-convexity}
		%	Let $f$ be $\mu$-strongly $g$-convex on $\mathcal{X}.$
		%\end{assumption}
		%\begin{assumption}[P\L-condition]\label{ass:pl-condition}
		%	Let $f$ be satisfy P\L-condition with constant $\mu.$
		%\end{assumption}
		\section{Riemannian LSVRG}
		\begin{algorithm}[t]
			\caption{Riemannian Loopless Stochastic Variance Reduced Gradient Descent (\algname{R-LSVRG})}\label{alg:RLSVRG}
			\begin{algorithmic}[1]
				\State	\textbf{Parameters:} initial point $x^0\in\mathcal{M},$ $y^0=x^0,$ learning rate $\eta > 0$
				\For{$k = 0, 1, 2, \ldots$}\do\\
				\State Sample $i\in[n]$ uniformly at random
				\State $g^{k} = \nabla f_i(x^k) - \Gamma_{y^k}^{x^k}\left( \nabla f_i(y^k) - \nabla f(y^k)\right)$
				\State $x^{k+1} = \mathrm{Exp}_{x^k}\left(-\eta g^k\right)$
				\State $y^{k+1} = 
				\begin{cases}
					x^k, & \text{with probability }p\\
					y^k, & \text{otherwise.}
				\end{cases}$
				\EndFor
			\end{algorithmic}
		\end{algorithm}
		
		In this section, we analyze the convergence guarantees of \algname{R-LSVRG} for solving the problem \eqref{eq:finite-sum}, where each $f_i$ ($i \in [n]$) is g-smooth and $f$ is strongly g-convex. In this context, we show that \algname{R-LSVRG} exhibits a linear convergence rate.
		\begin{algorithm*}
			\caption{Riemannian  ProbAbilistic Gradient Estimator (\algname{R-PAGE)}}\label{alg:RPAGE}
			\begin{algorithmic}[1]
				\State	\textbf{Parameters:}  $\text { initial point } x^0 \text {, stepsize } \eta>0 \text {, minibatch size } B, b<B \text {, probability } p \in(0,1]$
				\State\label{line:rpage-0-estimator} $g^0=\nabla f_{B}(x^0)$, \qquad \text{where }$\nabla f_{B}$ \text{is minibatch stochastic gradient estimator with size $B$} 
				\For{$k = 0, 1, 2, \ldots$}\do\\
				\State $x^{k+1} = \Expmap_{x^k}\left(-\eta g^k\right)$
				\State\label{line:rpage-estimator} $g^{k+1} = 
				\begin{cases}
					\nabla f_{B}(x^{k+1}), & \text{with probability }p\\
					\Gamma_{x^k}^{x^{k+1}}\left(g^k\right) + \nabla f_{b}(x^{k+1}) - \Gamma_{x^k}^{x^{k+1}}\left( \nabla f_{b}(x^k)\right), & \text{otherwise.}
				\end{cases}$
				\EndFor
			\end{algorithmic}
		\end{algorithm*}
		\begin{theorem}
			\label{thm:lsvrg_strongly_convex}
			Assuming that in \eqref{eq:finite-sum}, each $f_i$ is differentiable and $L$-g-smooth on $\mathcal{X}$ (Assumption~\ref{ass:l-g-smoothness} holds), and $f$ is $\mu$-strongly g-convex on $\mathcal{X}$ (Assumption~\ref{ass:g-strong-convexity} holds). Additionally, let Assumption~\ref{ass:geom} hold. Choose the stepsize $0<\eta \leq \frac{\mu}{16L^2\zeta}.$ Then, the iterates of the Riemannian L-SVRG method (Algorithm~\ref{alg:RLSVRG}) satisfy
			\begin{equation}\label{eq:convergence-rlsvrg-scvx}
				\mathbb{E}\left[\Phi^k \right]\leq \left(\max\left\lbrace 1 - \eta\mu, 1 - \frac{p}{2} \right\rbrace\right)^k\Phi^0,
			\end{equation}
			where $\Phi^k\eqdef d^2\left(x^k,x^{*}\right) + \frac{3\mu^2}{64pL^2\zeta}\mathcal{D}^k$ and  $\mathcal{D}^k\eqdef \norm{\Expmap^{-1}_{y^k}(x^{*})}^2.$
		\end{theorem}
		\begin{corollary}\label{cor:lsvrg_scvx_eps_solution}
			With the assumptions outlined in Theorem~\ref{thm:lsvrg_strongly_convex}, if we set $\eta = \frac{\mu}{16 L^2 \zeta}$, in order to obtain $\mathbb{E}[\Phi^k] \leq \varepsilon \Phi^0$, we find that the iteration complexity of Riemannian L-SVRG method (Algorithm~\ref{alg:RLSVRG}) is given by
			$$
			K = \mathcal{O}\left(\left(\frac{1}{p} + \frac{L^2 \zeta}{\mu^2}\right) \log \frac{1}{\varepsilon}\right).
			$$
		\end{corollary}
		
		By examining equation \eqref{eq:convergence-rlsvrg-scvx}, it becomes apparent that the contraction of the Lyapunov function is determined by $\max\{1-\eta\mu, 1- \frac{p}{2}\}$. Due to the constraint $\eta \leq  \frac{\mu}{16L^2\zeta}$, the first term is bounded from below by $1-\frac{\mu^2}{16 L^2 \zeta}$. Consequently, the complexity cannot surpass $\mathcal{O}\left(\frac{\zeta L^2}{\mu^2}\log \frac{1}{\varepsilon}\right)$. Regarding the total complexity, which is measured by the number of stochastic gradient calls, \algname{R-LSVRG}, on average, invokes the stochastic gradient oracle $2+p(n-1)$ times in each iteration. Combining these two complexities yields a total complexity of $\mathcal{O}\left(\left(\frac{1}{p} + n + \frac{\zeta L^2}{\mu^2}+ \frac{pn\zeta L^2}{\mu^2}\right) \log \frac{1}{\varepsilon}\right)$. It is noteworthy that choosing any $p \in \left[ \min \{\frac{c}{n}, \frac{c \mu^2}{L^2 \zeta}\} , \max \{\frac{c}{n}, \frac{c \mu^2}{L^2 \zeta}\}\right],$  where $c = \Theta(1)$, results in the optimal total complexity $\mathcal{O}\left(\left(n+ \frac{L^2 \zeta}{\mu^2}\right) \log \frac{1}{\varepsilon}\right)$. This resolution addresses a gap in \algname{R-SVRG} theory, where the length of the inner loop (in our case, $\frac{1}{p}$ on average) must be proportional to $\mathcal{O}\left(\zeta L^2 / \mu^2\right).$ Moreover, the analysis for \algname{R-LSVRG} is more straightforward and offers deeper insights.
		
		Let us succinctly formalize the aforementioned findings in the context of a corollary.
		
		\begin{corollary}\label{cor:rlsvrg-total-complexity}
			With the assumptions outlined in Theorem~\ref{thm:lsvrg_strongly_convex}, if we set $\eta = \frac{\mu}{16 L^2 \zeta}$ and $p = \frac{1}{n}$, in order to obtain $\mathbb{E}[\Phi^k] \leq \varepsilon \Phi^0$, we find that the total computational complexity of Riemannian L-SVRG method (Algorithm~\ref{alg:RLSVRG}) is given by
			$$
			K = \mathcal{O}\left(\left(n + \frac{L^2 \zeta}{\mu^2}\right) \log \frac{1}{\varepsilon}\right).
			$$
		\end{corollary}

		\section{Riemannian PAGE}
		\subsection{The Riemannian PAGE gradient estimator}
		
		The specific gradient estimator $g^{t+1}$ is formally defined in Line~\ref{line:rpage-estimator} of Algorithm \ref{alg:RPAGE}. It is essential to note that the method employs the simple stochastic gradient with a large minibatch size (resembling a full batch in a finite sum setting) with a small probability $p$, and with a substantial probability $1-p$, it utilizes the previous gradient $g^t$ with a minor adjustment involving the difference of stochastic gradients computed at two points, namely, $x^t$ and $x^{t+1}$ (a measure aimed at reducing computational costs, especially considering that $b \ll B$).
		
		In particular, when $p \equiv 1$, this aligns with vanilla minibatch Riemannian Stochastic Gradient Descent (\algname{R-SGD}), and further, when the minibatch size is set to $B=n$, it simplifies to Riemannian Gradient Descent (\algname{R-GD}). We present a straightforward formula for the optimal determination of $p$, expressed as $p \equiv \frac{b}{B+b}$. This formula proves sufficient for \algname{PAGE} to achieve optimal convergence rates, with additional nuances elucidated in the convergence results.
		\subsection{Convergence in Non-Convex Finite-Sum Setting}
		In this section, we delve into an analysis of the convergence guarantees offered by \algname{R-PAGE} in solving the problem \eqref{eq:finite-sum}, where each function $f_i$ (with $i \in [n]$) is characterized as g-smooth, and the function $f$ has a lower bound. Within this framework, we substantiate that \algname{R-PAGE} demonstrates a convergence rate aligning with current state-of-the-art results.
		
		\begin{theorem}\label{thm:page_noncvx}
			Assuming that each $f_i$ in \eqref{eq:finite-sum} is differentiable and $L$-g-smooth on $\mathcal{M}$ (Assumption~\ref{ass:l-g-smoothness} holds) and that $f$ is lower bounded on $\mathcal{M}$ (Assumption~\ref{ass:lower} holds), let $\delta^0 \stackrel{\text{def}}{=} f(x^0) - f^{*}$. Select the stepsize $\eta$ such that
			$0 < \eta \leq \frac{1}{L\left(1 + \sqrt{\frac{1-p}{pb}}\right)},$ and the minibatch size $B = n,$ and the secondary minibatch size $b < B.$ 
			Then, the iterates of the \algname{R-PAGE} method (Algorithm~\ref{alg:RPAGE}) satisfy
			\begin{equation*}
				\mathbb{E}\left[\norm{\nabla f(\widehat{x}^K)}^2\right] \leq\frac{2\delta^0}{\eta K},
			\end{equation*}
			where $\widehat{x}^k$ is chosen randomly from $x^0,\ldots,x^{K-1}$ with uniform probability distribution.
		\end{theorem}
		
		\begin{corollary}\label{cor:page_noncvx}
			With the assumptions outlined in Theorem~\ref{thm:page_noncvx}, if we set $\eta = \frac{1}{L\left(1 + \sqrt{\frac{1-p}{pb}}\right)}$, in order to obtain $\mathbb{E}\left[\left\|\nabla f\left(\widehat{x}^K\right)\right\|\right] \leq \varepsilon$, we find that the iteration complexity of \algname{R-PAGE} method (Algorithm~\ref{alg:RPAGE}) is given by
			$$
			K = \mathcal{O}\left(\left(1+\sqrt{\frac{1-p}{pb}}\right) \frac{L\delta^0}{\varepsilon^2}\right).
			$$
		\end{corollary}
		
		Additionally, in accordance with \algname{R-PAGE} gradient estimator (refer to Line~\ref{line:rpage-estimator} of Algorithm \ref{alg:RPAGE}), it is evident that, on average, it employs $pB+(1-p)b$ stochastic gradients for each iteration. Consequently, the computational load for stochastic gradient computations, denoted as $\# \operatorname{grad}$ (i.e., gradient complexity), can be expressed as: $
		\# \operatorname{grad} = B + K\left(p B + (1-p) b\right) = \mathcal{O}\left( B + \frac{2 \delta_0 L}{\varepsilon^2}\left(1 + \sqrt{\frac{1-p}{p b}}\right)\left(p B + (1-p) b\right)\right).
		$
		
		It is important to highlight that the initial $B$ in $\# \operatorname{grad}$ accounts for the computation of $g^0$ (refer to Line~\ref{line:rpage-0-estimator} in Algorithm \ref{alg:RPAGE}). Subsequently, we present a parameter configuration that yields the best known convergence rate for the non-convex finite-sum problem \eqref{eq:finite-sum}.
		
		\begin{corollary}
			\label{cor:page_grad_computes}
			With the assumptions outlined in Theorem~\ref{thm:page_noncvx}, if we set  $\eta = \frac{1}{L\left(1 + \sqrt{\frac{1-p}{pb}}\right)}$, minibatch size $B=n$, secondary minibatch size $b \leq \sqrt{B}$ and $p \equiv \frac{b}{B+b}$, in order to obtain  $\mathbb{E}\left[\left\|\nabla f\left(\widehat{x}^K\right)\right\|\right] \leq \varepsilon$, we find that the total computational complexity of Riemannian PAGE method (Algorithm~\ref{alg:RPAGE}) is given by
			$$
			\#	\operatorname{grad} =  \mathcal{O}\left(n + \frac{\sqrt{n}}{\varepsilon^2}\right).
			$$
		\end{corollary}
		
		In the subsequent discussion, we delve into an analysis within a non-convex finite-sum setting under the Polyak-\L ojasiewicz condition. This assumption allows us to assert a linear convergence rate.
		\begin{theorem}\label{thm:rpage-pl-condition}
			Assuming that each $f_i$ in \eqref{eq:finite-sum} is differentiable and $L$-g-smooth on $\mathcal{M}$ (Assumption~\ref{ass:l-g-smoothness} holds), also that $f$ is lower bounded on $\mathcal{M}$ (Assumption~\ref{ass:lower} holds) and it satisfies Polyak-\L ojasiewicz condition (Assumption~\ref{ass:pl-condition} holds) with constant $\mu$, let $\delta^0 \stackrel{\text{def}}{=} f(x^0) - f^{*}$. Select the stepsize $\eta$ such that $	\eta \leq \min \left\{ \frac{1}{L \left(1 + \sqrt{\frac{1-p}{pb} }\right)}, \frac{p}{2\mu} \right\}.$ Then, the iterates of the Riemannian PAGE method (Algorithm \ref{alg:RPAGE}) satisfy 
			$$\mathbb{E}\left[\Phi^{K}\right] \leq(1-\mu \eta)^K \delta^0,$$
			where $\Phi^k\eqdef f\left(x^k\right)-f^*+\frac{2}{p}\left\|g^k-\nabla f\left(x^k\right)\right\|^2$ and $\Phi^0 = \delta^0$, since $g^0 = \nabla f(x^0).$
		\end{theorem}
		\begin{corollary}\label{cor:rlsvrg-iteration-complexity}
			With the assumptions outlined in Theorem~\ref{thm:rpage-pl-condition}, if we set $\eta = \min \left\{ \frac{1}{L \left(1 + \sqrt{\frac{1-p}{pb} }\right)}, \frac{p}{2\mu} \right\}$, in order to obtain $\mathbb{E}[\Phi^k] \leq \varepsilon \Phi^0$, we find that the iteration complexity of Riemannian \algname{R-PAGE} method (Algorithm~\ref{alg:RPAGE}) is given by
			$$
			K =\mathcal{O}\left( \left( \left(1 + \sqrt{\frac{1-p}{pb}}\right) \frac{L}{\mu} + \frac{1}{p} \right) \log \frac{1}{\varepsilon}\right).
			$$
		\end{corollary}

		Similarly to general non-convex setting we can estimate the computational load for stochastic gradient computations:
		$
		\# \operatorname{grad} = B + K\left(p B + (1-p) b\right) = \mathcal{O}\left( B + \frac{2 \delta_0 L}{\varepsilon^2}\left(1 + \sqrt{\frac{1-p}{p b}}\right)\left(p B + (1-p) b\right)\right).
		$ 
		Following that, we introduce a parameter setup that produces the most favorable known convergence rate for the Riemannian non-convex finite-sum problem~\eqref{eq:finite-sum}.
		
		\begin{corollary}\label{cor:rpage-pl-total-cxty}
			
			With the assumptions outlined in Theorem~\ref{thm:rpage-pl-condition}, if we set  $\eta =  \min \left\{ \frac{1}{L \left(1 + \sqrt{\frac{1-p}{pb} }\right)}, \frac{p}{2\mu} \right\}$, minibatch size $B=n$, secondary minibatch size $b \leq \sqrt{B}$ and $p \equiv \frac{b}{B+b}$, in order to obtain $\mathbb{E}[\Phi^k] \leq \varepsilon \Phi^0$, we find that the total computational complexity of Riemannian PAGE method (Algorithm~\ref{alg:RPAGE}) is given by
			$$
			\#	\operatorname{grad} = \mathcal{O} \left( \left(n + \frac{L}{\mu}\sqrt{n}\right) \log \frac{1}{\varepsilon} \right).
			$$
		\end{corollary}
		
		\subsection{Convergence in Non-Convex Online Setting}
		In this section, our attention is directed towards non-convex online problems, specifically denoted as 
		\begin{equation}\label{eq:online-problem}
			\min_{x\in\mathcal{M}}\left\lbrace f(x) =\mathbb{E}_{\xi\sim\mathcal{D}}\left[\nabla f\left(x, \xi\right)\right]  \right\rbrace.
		\end{equation}
		It is noteworthy to recall that we characterize this online problem \eqref{eq:online-problem} as an extension of the finite-sum problem \eqref{eq:finite-sum} with a substantial or infinite $n$. The consideration of the bounded variance assumption (Assumption \ref{ass:bounded-variance}) and unbiasedness (Assumption \ref{ass:unbiasedness}) are imperative in this online scenario. Analogously, we begin by presenting the primary theorem in this online context, followed by corollaries outlining the optimal convergence outcomes. 	
		
		%We employ the following notation: $\nabla f_b(x) \eqdef \frac{1}{b}\sum_{i\in I}\nabla f(x, \xi_i),$ where $|I|=b.$
		We define the notation $\nabla f_{B}(x)$ as the average gradient over a subset $I$ of size $B$: $$\nabla f_{B}(x) \eqdef \frac{1}{B}\sum_{i\in I}\nabla f(x, \xi_i), \text{ with } |I|=B.$$

		\begin{assumption}[Bounded variance]\label{ass:bounded-variance}
			The stochastic gradient $\nabla f_{B}(x)$ has bounded variance if there exists $\sigma \geq 0$ such that $\mathbb{E}\|\nabla f(x, \xi_i) - \nabla f(x)\|^2\leq\sigma^2,$ for all $x\in\mathcal{M}.$
		\end{assumption}
		
		\begin{assumption}[Unbiasedness]
			\label{ass:unbiasedness}
			The stochastic gradient estimator is unbiased if $\mathbb{E}\left[\nabla f_{B}(x)\right] = \nabla f(x).$ 
		\end{assumption}
		
		Let us formulate a theorem for the general non-convex online setting.
		\begin{algorithm*}
			\caption{Riemannian MARINA (\algname{R-MARINA)}}\label{alg:R-MARINA}
			\begin{algorithmic}[1]
				\State \textbf{Input:} initial point $x^0\in \mathbb{R}^d$, stepsize $\eta>0$, probability ${p} \in (0, 1],$ 
				\State $g^0 = \nabla f(x^0)$
				\For{$k =0,1,\dots,K-1$}
				%	\State Sample $c_k\sim\mathrm{Bern}(p)$
				\State Broadcast $g^k$ to all workers
				\For{$i=1,\ldots,n$ in parallel}
				\State $x^{k+1} = \mathrm{Exp}_{x^k}\left(-\eta g^k\right)$
				\State $g_i^{k+1} = \begin{cases}
					\nabla f_i(x^{k+1}) & \text{with probability $p$}\\
					\Gamma_{x^k}^{x^{k+1}}g_i^k + \mathcal{Q}_i\left(\nabla f_i(x^{k+1}) - \Gamma_{x^k}^{x^{k+1}}\nabla f_i(x^{k})\right) & \text{otherwise}
				\end{cases}$
				\EndFor
				\State Receive $g_i^{k+1}$ from all workers
				\State $g^{k+1} = \frac{1}{n}\sum_{i=1}^{n}g_i^{k+1}$
				\EndFor
			\end{algorithmic}
		\end{algorithm*}
		\begin{theorem}\label{thm:rpage-noncvx-online}
			Suppose that Assumption~\ref{ass:bounded-variance} and Assumption~\ref{ass:unbiasedness} hold for the stochastic gradient $\nabla f_{B}(x)$ and suppose that, for any fixed $\xi_i,$ $f(x,\xi_i)$ is geodesically $L$-smooth on $\mathcal{M}$ (Assumption \ref{ass:l-g-smoothness} holds), also that $f$ is lower bounded on $\mathcal{M}$ (Assumption~\ref{ass:lower} holds). Let $\delta^0 \stackrel{\text{def}}{=} f(x^0) - f^{*}$. Select the stepsize $\eta$ such that $		\eta \leq \frac{1}{L \left( 1 + \sqrt{\frac{1-p}{pb}} \right)}.$ Then, the iterates of the Riemannian PAGE method (Algorithm \ref{alg:RPAGE}) satisfy 
			$$
			\mathbb{E}\|\nabla f(\widehat{x}^K)\|^2 \leq \frac{2\mathbb{E}[\Phi^0]}{\eta K} + \frac{\sigma^2}{B}.
			$$
			where $\widehat{x}^{K}$ is chosen uniformly at random from $x^0, \dots, x^{K-1}$ and $ \Phi^0 = f\left(x^0\right)-f^*+\frac{\eta \sigma^2}{2 p B}$
		\end{theorem}
		
		\begin{corollary}\label{cor:rpage-iteration-cxty-online}
			With the assumptions outlined in Theorem~\ref{thm:rpage-noncvx-online}, if we set $\eta =  \frac{1}{L \left(1 + \sqrt{\frac{1-p}{pb} }\right)}$, in order to obtain $\mathbb{E}[\Phi^k] \leq \varepsilon \Phi^0$, we find that the iteration complexity of Riemannian R-PAGE method (Algorithm~\ref{alg:RPAGE}) is given by
			$$
			K =\mathcal{O}\left(  \frac{\delta^0L}{\varepsilon^2} \left( 1 + \sqrt{\frac{1-p}{pb}} \right)  + \frac{1}{p}\right).
			$$
		\end{corollary}
		We present a parameter configuration that yields the most favorable convergence rate known for the Riemannian non-convex online problem \eqref{eq:online-problem}.
		\begin{corollary}
			\label{cor:rpage-noncvx-online-optimality}
			With the assumptions outlined in Theorem~\ref{thm:rpage-noncvx-online}, if we set $\eta =  \frac{1}{L \left(1 + \sqrt{\frac{1-p}{pb} }\right)}$, minibatch size $B = \left\lceil\frac{2\sigma^2}{\varepsilon^2}\right\rceil,$ secondary minibatch size $b \leq \sqrt{B}$ and probability $p = \frac{b}{B+b},$ in order to obtain $\mathbb{E}[\Phi^k] \leq \varepsilon \Phi^0$, we find that the total computational complexity of Riemannian PAGE method of Riemannian R-PAGE method (Algorithm~\ref{alg:RPAGE}) is given by
			$$
			\#grad =  \mathcal{O}\left(\frac{\sigma^2}{\varepsilon^2} + \frac{\sigma}{\varepsilon^3}\right).
			$$
		\end{corollary}
		We examine an analysis in a non-convex online scenario, considering the Polyak-Łojasiewicz condition. Due to space constraints, we are relocating this discussion to the appendix.

		\begin{figure*}[ht!]
			\centering
			%	\captionsetup[sub]{font=scriptsize,labelfont={}}	
			
			\includegraphics[width=0.48\textwidth]{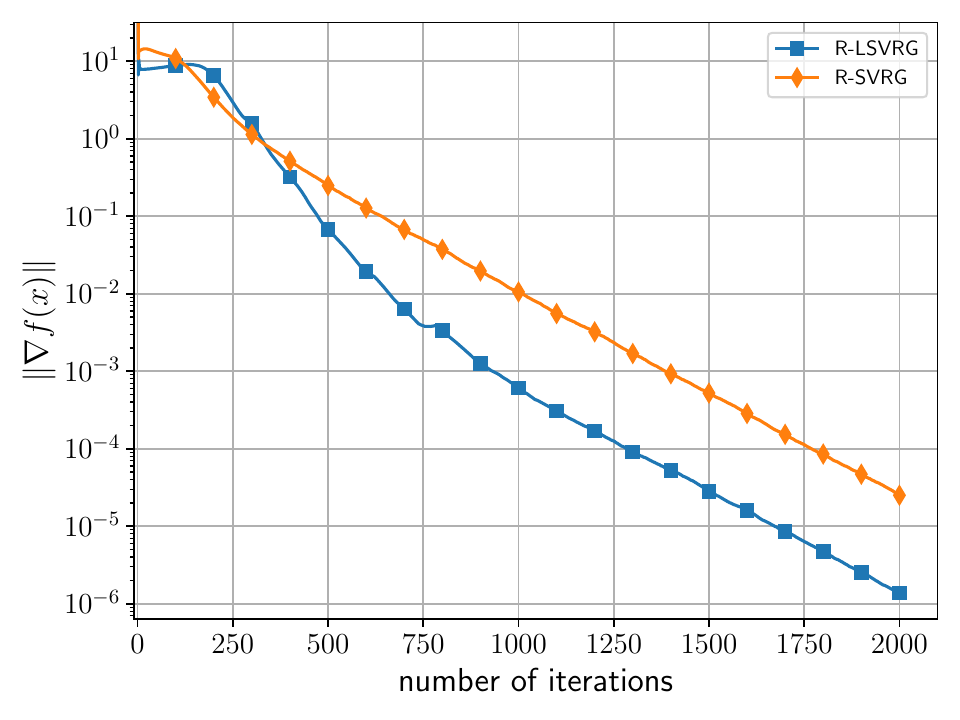}
			\includegraphics[width=0.48\textwidth]{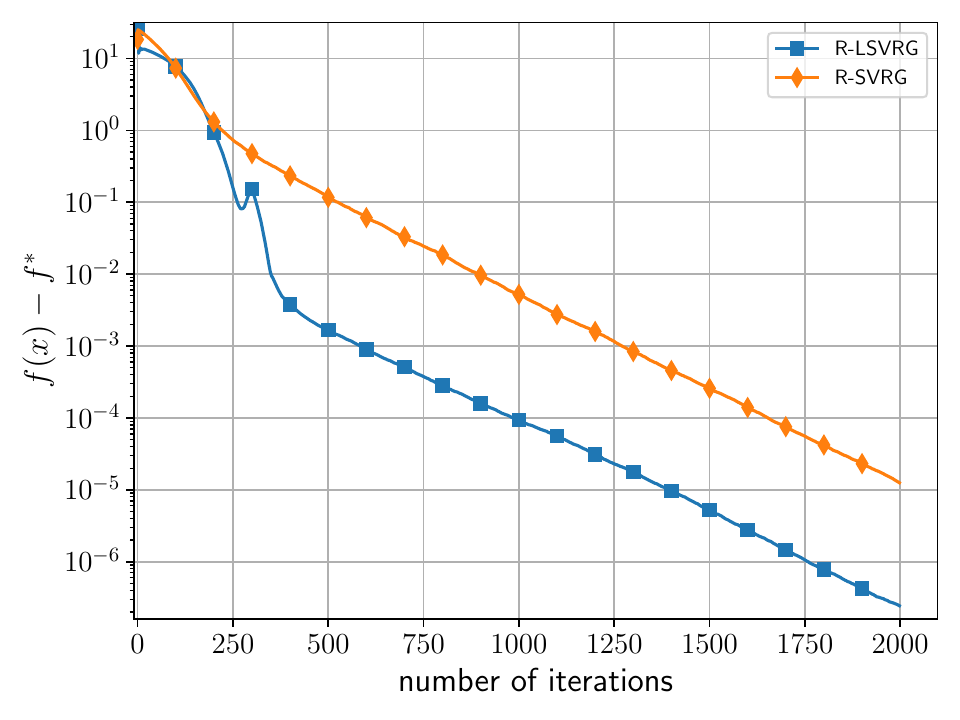}
			\caption{\small{Comparison of the \algname{R-LSVRG} and \algname{R-SVRG} methods: Left panel illustrates convergence in terms of gradient norm, while the right panel depicts convergence in terms of function values.}}
			\label{fig:exp1_th_step sizes}
		\end{figure*}
		\section{Riemannian MARINA}
		
		In this section, we consider the finite-sum problem \eqref{eq:finite-sum} as a distributed optimization problem, where each function $f_i(x)$ is stored or located on different workers/machines that collaborate to achieve a common objective. Instead of centralizing the optimization process, it is parallel, with each component handling a portion of the computation. This approach is often employed in large-scale systems, such as distributed computing networks, parallel processing, or decentralized machine learning algorithms. The goal is to enhance efficiency, scalability, and the ability to handle complex tasks by leveraging the computational resources of multiple entities.
		
		We outline another algorithm employed in our study: Riemannian \algname{MARINA} (refer to Algorithm \ref{alg:R-MARINA}). During each iteration of \algname{R-MARINA}, each worker i randomly selects between sending the dense vector $\nabla f_i(x^{k+1}) $ with a probability of $p$ or sending the compressed gradient difference $ \mathcal{Q}_i\left(\nabla f_i(x^{k+1}) - \Gamma_{x^k}^{x^{k+1}}\nabla f_i(x^{k})\right) $ with a probability of $1-p$. In the former scenario, the server simply averages the vectors received from workers to obtain $g^{k+1}=\nabla f\left(x^{k+1}\right)$. Conversely, in the latter case, the server averages the compressed differences from all workers and adds the result to $g^k$, yielding $g^{k+1}$. 
		
		Next, we delineate an extensive category of unbiased compression operators that adhere to a specific variance constraint.
		\begin{definition}\label{def:quantization}
			We say that a stochastic mapping $Q: T_x\mathcal{M} \rightarrow T_x\mathcal{M},$ for all $x\in\mathcal{M},$ is a unbiased compression operator/compressor with conic variance if there exists $\omega \geq 0$ such that for any $x \in \mathcal{M},$ any $y\in T_x\mathcal{M},$ we have
			\begin{equation*}
				\mathbb{E}\left[Q(y)\right] = y, \quad 	\mathbb{E}\left[\|Q(y) - y\|^2\right] \leq \omega \|y\|^2.
			\end{equation*}
			For the given compressor $Q(y),$ we define the expected density as $\rho_Q = \sup_{x\in\mathcal{M}, y \in T_x\mathcal{M}} 	\mathbb{E}\left[\|Q(y)\|_0\right],$ where $ \|z\|_0$ is the number of non-zero components of $z \in T_x\mathcal{M}.$
		\end{definition}
		Notice that the expected density is well-defined for any compression operator since $\|Q(z)\|_0\leq d,$ where $d$ is the local dimension of the manifold. Next, we present the optimal communication complexity for the Riemannian distributed non-convex setting.

		\begin{theorem}\label{thm:marina_noncvx}
			Assume each $f_i$ is $L$-g-smooth on $\mathcal{M}$ (Assumption~\ref{ass:l-g-smoothness} holds) and let $f$ be uniformly lower bounded on $\mathcal{M}$ (Assumption~\ref{ass:lower} holds). Assume that the compression operator is unbiased and has conic variance (Definition \ref{def:quantization}). Let $\delta^0 \stackrel{\text{def}}{=} f(x^0) - f^{*}$. Select the stepsize $\eta$ such that $		\eta \leq \frac{1}{L \left( 1 + \sqrt{\frac{1-p}{p}\frac{\omega}{n}} \right)}.$ Then, the iterates of the Riemannian MARINA method (Algorithm \ref{alg:R-MARINA}) satisfy 
			\begin{equation}\label{eq:rmarina_convergence_guarantees}
				\mathbb{E}\left[\|f\left(\widehat{x}^{K}\right)\|^2\right] \leq\frac{2 \delta^0}{\gamma K},
			\end{equation}
			where $\widehat{x}^{K}$ is chosen uniformly at random from $x^0, \dots, x^{K-1}$.
		\end{theorem}
		
		\begin{corollary}\label{cor:rmarina_noncvx}
			With the assumptions outlined in Theorem~\ref{thm:marina_noncvx}, if we set $\eta =  \frac{1}{L \left(1 + \sqrt{\frac{1-p}{p}\frac{\omega}{n} }\right)}$, in order to obtain 	$\mathbb{E}\left[\|\nabla f\left(\widehat{x}^{K}\right)\|^2\right] \leq \varepsilon^2,$ we find that the communication complexity of Riemannian MARINA method (Algorithm~\ref{alg:R-MARINA}) is given by
			$$
			K = \mathcal{O} \left( \frac{\delta^0 L}{\varepsilon^2} \left(1 + \sqrt{\frac{(1-p) \omega}{pn}}\right) \right).
			$$
		\end{corollary}

		\begin{corollary}\label{cor:rmarina_noncvx_total}
			With the assumptions outlined in Theorem~\ref{thm:marina_noncvx}, if we set $\eta =  \frac{1}{L \left(1 + \sqrt{\frac{1-p}{p}\frac{\omega}{n} }\right)}$, $p = \frac{\rho_{\mathcal{Q}}}{d}$ in order to obtain 	$\mathbb{E}\left[\|\nabla f\left(\widehat{x}^{K}\right)\|^2\right] \leq \varepsilon^2,$ we find that the communication complexity $\mathcal{C}$ of Riemannian MARINA method (Algorithm~\ref{alg:R-MARINA}) is given by
			$$
			\mathcal{C} = \mathcal{O}\left(\frac{\delta_0 L}{\varepsilon^2}\left(1+\sqrt{\frac{\omega}{n}\left(\frac{d}{\rho_{\mathcal{Q}}}-1\right)}\right)\right),
			$$
			where $\rho_{Q}$ is the expected density of the quantization (see Def.~\ref{def:quantization}) and the expected total communication cost per worker is $\mathcal{O}\left(d+\rho_{\mathcal{Q}} K\right)$.
		\end{corollary}
		This signifies that, to the best of our knowledge, we have attained the initial findings for the distributed Riemannian non-convex setting with gradient compression. Our results are consistent with the optimal communication complexity observed in the Euclidean case.
		\section{Experiments}
		In this section, we utilize our examination of \algname{R-LSVRG} to address the problem of fast computation of eigenvectors - a fundamental challenge currently under active investigation in the big-data context \citep{garber2015fast, jin2015robust, shamir2015stochastic}. Specifically, we have the following problem. 
		$$\min _{x^{\top} x=1}-x^{\top}\left(\sum_{i=1}^n z_i z_i^{\top}\right) x \eqdef -x^{\top} A x=f(x),$$
		In our experiments, we employ the Pymanopt toolbox for Riemannian optimization in Python \citep{JMLR:v17:16-177}. We conduct a comparative analysis between the \algname{R-LSVRG} method (Algorithm \ref{alg:RLSVRG}) and the loop version, which is Riemannian \algname{SVRG} method. To ensure a fair comparison, we set identical values for the step size, denoted as $\eta$, and the parameter for the inner loop, fixed at $n$ (for \algname{R-LVSRG}, we utilize $p=1/n$). The results demonstrate a significant improvement of \algname{R-LSVRG} over \algname{R-SVRG}.
		
		\section{Conclusion}
		In this study, we examined the loopless variant of the well-known Riemannian \algname{SVRG} method and demonstrated that replacing the double loop structure with a coin flip mechanism enhances the method both theoretically and empirically. Motivated by this observation, we extended the PAGE method to Riemannian non-convex settings. Our results align with the current state-of-the-art theoretical guarantees, yet they indicate that the analysis can be significantly simplified and made more accessible. Leveraging the \algname{R-PAGE} method as a foundation, we also obtained, to the best of our knowledge, the initial findings for the distributed Riemannian setting with communication gradient compression through the application of the \algname{R-MARINA} method. Nevertheless, our analysis does not encompass acceleration mechanisms, a topic deserving a dedicated paper. This aspect is deferred to future work, presenting an avenue of high interest for the scientific community.

		\bibliography{bibliography}

		%%%%%%%%%%%%%%%%%%%%%%%%%%%%%%%%%%%%%%%%%%%%%%%%%%%%%%%%%%%%%%%%%%%%%%%%%%%%%%%
		%%%%%%%%%%%%%%%%%%%%%%%%%%%%%%%%%%%%%%%%%%%%%%%%%%%%%%%%%%%%%%%%%%%%%%%%%%%%%%%
		% APPENDIX
		%%%%%%%%%%%%%%%%%%%%%%%%%%%%%%%%%%%%%%%%%%%%%%%%%%%%%%%%%%%%%%%%%%%%%%%%%%%%%%%
		%%%%%%%%%%%%%%%%%%%%%%%%%%%%%%%%%%%%%%%%%%%%%%%%%%%%%%%%%%%%%%%%%%%%%%%%%%%%%%%
		\newpage
		\appendix
		\section{Extended Related Work}
		
		\subsection{Riemannian Optimization}
		
		References predating this point can be located in works of \citet{udriste1997optimization, udriste2013convex}. Stochastic Riemannian optimization has been explored in prior work \citet{liu2003optimal}, albeit with a focus on asymptotic convergence analysis and an absence of specific convergence rates. \citet{boumal2019global} conducted an analysis of the iteration complexity of Riemannian trust-region methods, whereas \citet{bento2017iteration} delved into the non-asymptotic convergence of Riemannian gradient, subgradient, and proximal point methods. \citet{tripuraneni2018averaging} investigated aspects beyond variance reduction to expedite the convergence of first-order optimization methods on Riemannian manifolds. In the context of Federated Learning, the Riemannian optimization methods are applicable, as discussed by \citet{Li2022FederatedLO}. Accelerating methods are also widely analysied in Riemannian setting \citep{pmlr-v167-martinez-rubio22a, pmlr-v195-martinez-rubio23a}

		\subsection{Variance Reduction}
		Reducing variance in stochastic optimization involves employing various techniques. Control variates, a common method used in Monte Carlo simulations, are among the widely adopted variance reduction techniques \citep{haghighat2003monte, rubinstein2016simulation}. There has been a notable upswing in interest recently regarding variance-reduced methods for solving finite-sum problems in linear spaces \citep{j2015variance, hanzely2018sega, malinovsky2022federated, hanzely2020variance, shang2018vr, malinovsky2022variance}. Furthermore, there is extensive analysis of acceleration techniques incorporating Nesterov momentum \citep{nesterov1983method} with variance reduction \citep{allen2018katyusha, allen2018katyusha_1, qian2021svrg, xu2021katyusha, condat2023tamuna, grudzien2023improving}. In the context of sampling without replacement (random reshuffling), researchers have also investigated variance reduction \citep{ying2020variance, huang2021improved, malinovsky2023random}. In the realm of non-convex optimization, researchers actively explore a multitude of strategies and methodologies for mitigating variance, a crucial aspect in the optimization process \citep{cutkosky2019momentum, liu2017variance, kavis2022adaptive}.
		\subsection{Communication Compression}
		In a distributed setting, various compression schemes are extensively studied to enhance the communication efficiency of methods \citep{Seide20141bitSG, terngrad, ivkin2019communication, tang2021communication, sattler2019sparse}. The investigation and incorporation of variance reduction mechanisms into compression and acceleration techniques represent an area of extensive research. Researchers are actively exploring diverse strategies and methodologies aimed at refining compression \citep{sadiev2022federated} and acceleration techniques \citep{li2020acceleration, yang2021fastsgd}. In the context of biased compression, the error feedback mechanism \citep{karimireddy2019error} is extensively analyzed and has been demonstrated to be effectively combined with acceleration \citep{qian2021error}. In a non-convex setting, variance reduction mechanisms are explored within the \algname{MARINA} method, as discussed by \citet{gorbunov2021marina}. Additionally, several enhancements derived from the \algname{MARINA} method are presented in several works \citep{szlendak2021permutation, PanferovCorQuant, tyurin2022dasha}. In the context of Byzantine robustness, the \algname{MARINA} method has been thoroughly analyzed \citep{gorbunov2022variance, malinovsky2023byzantine, rammal2023communication}.

		\clearpage

		\section{R-LSVRG: Strongly Convex Case}
		\noindent\textbf{Theorem~\ref{thm:lsvrg_strongly_convex}.} \textit{Assuming that in \eqref{eq:finite-sum}, each $f_i$ is differentiable and $L$-g-smooth on $\mathcal{X}$ (Assumption~\ref{ass:l-g-smoothness} holds), and $f$ is $\mu$-strongly g-convex on $\mathcal{X}$ (Assumption~\ref{ass:g-strong-convexity} holds). Additionally, let Assumption~\ref{ass:geom} hold. Choose the stepsize $0<\eta \leq \frac{\mu}{16L^2\zeta}.$ Then, the iterates of the Riemannian L-SVRG method (Algorithm~\ref{alg:RLSVRG}) satisfy
			\begin{equation*}
				\mathbb{E}\left[\Phi^k \right]\leq \left(\max\left\lbrace 1 - \eta\mu, 1 - \frac{p}{2} \right\rbrace\right)^k\Phi^0,
			\end{equation*}
			where $\Phi^k\eqdef d^2\left(x^k,x^{*}\right) + \frac{3\mu^2}{64pL^2\zeta}\mathcal{D}^k$ and  $\mathcal{D}^k\eqdef \norm{\Expmap^{-1}_{y^k}(x^{*})}^2.$}
		\begin{proof}[Proof of Theorem~\ref{thm:lsvrg_strongly_convex}]
			Conditioned on $x^k,$ take expectation with respect to $i:$
			\begin{equation*}
				\begin{split}
					\Exp{\norm{g^k}^2} &= \Exp{\norm{\frac{1}{B}\sum_{i\in I}\left(\nabla f_i(x^k) - \Gamma_{y^k}^{x^k}\nabla f_i(y^k)\right) + \Gamma_{y^k}^{x^k}\nabla f(y^k)}^2}\\
					&=\Exp{\norm{\frac{1}{B}\sum_{i\in I}\left(\nabla f_i(x^k) - \Gamma_{y^k}^{x^k}\nabla f_i(y^k)\right) + \Gamma_{y^k}^{x^k}\left(\nabla f(y^k) - \Gamma_{x^{*}}^{y^k}\nabla f(x^{*})\right)}^2}\\
					&\leq 2 \Exp{\norm{\frac{1}{B}\sum_{i\in I}\left(\nabla f_i(x^k) - \Gamma_{y^k}^{x^k}\nabla f_i(y^k)\right)}^2}+ 2\Exp{\norm{\nabla f(y^k) - \Gamma_{x^{*}}^{y^k}\nabla f(x^{*})}^2}\\
					&\leq 2L^2\norm{\Expmap^{-1}_{x^k}(y^k)}^2 + 2L^2\norm{\Expmap^{-1}_{y^k}(x^{*})}^2\\
					&\leq 2L^2\left(\norm{\Expmap^{-1}_{x^k}(x^{*})}+\norm{\Expmap^{-1}_{y^k}(x^{*})}\right)^2 + 2L^2\norm{\Expmap^{-1}_{y^k}(x^{*})}^2\\
					&\leq 4L^2\norm{\Expmap^{-1}_{x^k}(x^{*})}^2 + 6L^2\norm{\Expmap^{-1}_{y^k}(x^{*})}^2.
				\end{split}
			\end{equation*}
			We use the definition of $g^k$ from Algorithm~\ref{alg:RLSVRG}, the fact that $\nabla f(x^{*})=0,$ the Young's inequality and the fact that the parallel transport preserves the scalar product, $L$-g-smoothness of $f(x),$ the triangle inequality, the Young's inequality.
			
			Let us define $\mathcal{D}^k\eqdef \norm{\Expmap^{-1}_{y^k}(x^{*})}^2.$ Notice that
			\begin{equation}\label{eq:recursive_bound}
				\begin{split}
					\Exp{\mathcal{D}^{k+1}} &= p\norm{\Expmap^{-1}_{x^k}(x^{*})}^2 + \left(1-p\right)\norm{\Expmap^{-1}_{y^k}(x^{*})}^2\\
					& = pd^2(x^k,x^{*}) + (1-p)\mathcal{D}^{k}.\\
				\end{split}
			\end{equation}
			
			Further, we have
			\begin{equation*}
				\begin{split}
					\Exp{d^2(x^{k+1}, x^{*})}&\leq d^2(x^k,x^{*}) + 2\eta\langle \Expmap^{-1}_{x^k}(x^{*}),\Exp{g^k} \rangle + \zeta\eta^2\Exp{\norm{g^k}^2}\\
					&\leq d^2(x^k,x^{*}) + 2\eta\langle \Expmap^{-1}_{x^k}(x^{*}),\nabla f(x^k) \rangle\\
					& + \zeta\eta^2\left(4L^2\norm{\Expmap^{-1}_{x^k}(x^{*})}^2 + 6L^2\norm{\Expmap^{-1}_{y^k}(x^{*})}^2\right)\\
					&\leq \left(1 - \eta\mu + 4L^2\eta^2\zeta\right)d^2(x^k,x^{*}) + 2\eta\left(f(x^{*}) - f(x^k)\right)\\
					& + 6L^2\eta^2\zeta\mathcal{D}^k.
				\end{split}
			\end{equation*}
			We use Corollary~\ref{th:distance_corollary}, bound on $\Exp{\norm{g^k}^2},$ strong convexity.
			
			Observe that due to $g$-strong convexity we also have that 
			\[f(x^k) - f(x^{*})\geq \langle \nabla f(x^{*}), \Expmap^{-1}_{x^k}(x^{*})\rangle - \frac{\mu}{2}d^2\left(x^k,x^{*}\right) = - \frac{\mu}{2}d^2\left(x^k,x^{*}\right).\]
			Therefore, we obtain that
			\begin{equation*}
				\begin{split}
					\Exp{d^2(x^{k+1}, x^{*})}&\leq \left(1 - 2\eta\mu + 4L^2\eta^2\zeta\right)d^2(x^k,x^{*}) + 6L^2\eta^2\zeta\mathcal{D}^k.
				\end{split}
			\end{equation*}
			Multiply both sides of \eqref{eq:recursive_bound} by $\frac{12L^2\eta^2\zeta}{p},$ and add it to the last equation:
			\begin{equation*}
				\begin{split}
					\Exp{d^2(x^{k+1}, x^{*})} + \frac{12L^2\eta^2\zeta}{p}\Exp{\mathcal{D}^{k+1}} &\leq \left(1 - 2\eta\mu + 16L^2\eta^2\zeta\right)d^2(x^k,x^{*})\\
					& + \frac{12L^2\eta^2\zeta}{p}\left(1 - \frac{p}{2}\right)\mathcal{D}^k.
				\end{split}
			\end{equation*}
			Define the Lyapunov funciton $\Phi^k\eqdef d^2\left(x^k,x^{*}\right) + \frac{12L^2\eta^2\zeta}{p}\mathcal{D}^k.$ Then we have that
			\begin{equation*}
				\Exp{\Phi^{k+1}}\leq \max\left\lbrace 1 - 2\eta\mu + 16L^2\eta^2\zeta, 1 - \frac{p}{2} \right\rbrace\Phi^k.
			\end{equation*}
			Due to the choice of $\eta = \frac{\mu}{16L^2\zeta},$ it follows that $1 - 2\eta\mu + 16L^2\eta^2\zeta = 1 - \frac{\mu^2}{16L^2\zeta}.$ Unrolling the recursion, we arrive at
			\begin{equation*}
				\Exp{\Phi^k}\leq \left(\max\left\lbrace 1 - \frac{\mu^2}{16L^2\zeta}, 1 - \frac{p}{2} \right\rbrace\right)^k\Phi^0.
			\end{equation*}
		\end{proof}
		
		\noindent\textbf{Corollary~\ref{cor:lsvrg_scvx_eps_solution}.} \textit{With the assumptions outlined in Theorem~\ref{thm:lsvrg_strongly_convex}, if we set $\eta = \frac{\mu}{16 L^2 \zeta}$, in order to obtain $\mathbb{E}[\Phi^k] \leq \varepsilon \Phi^0$, we find that the iteration complexity of Riemannian L-SVRG method (Algorithm~\ref{alg:RLSVRG}) is given by
			$$
			K = \mathcal{O}\left(\left(\frac{1}{p} + \frac{L^2 \zeta}{\mu^2}\right) \log \frac{1}{\varepsilon}\right).
			$$}
		\begin{proof}[Proof of Corollary~\ref{cor:lsvrg_scvx_eps_solution}]
			Observe that
			\begin{equation*}
				\left(1 - \frac{\mu^2}{16L^2\zeta}\right)^k\leq e^{-\frac{k\mu^2}{16L^2\zeta}}, \quad \left(1 - \frac{p}{2}\right)^k \leq e^{-\frac{kp}{2}}.
			\end{equation*}
			Now the result simply follows from~\eqref{eq:convergence-rlsvrg-scvx}.
		\end{proof}
		
		\noindent\textbf{Corollary~\ref{cor:rlsvrg-total-complexity}.} \textit{With the assumptions outlined in Theorem~\ref{thm:lsvrg_strongly_convex}, if we set $\eta = \frac{\mu}{16 L^2 \zeta}$ and $p = \frac{1}{n}$, in order to obtain $\mathbb{E}[\Phi^k] \leq \varepsilon \Phi^0$, we find that the total computational complexity of Riemannian \algname{L-SVRG} method (Algorithm~\ref{alg:RLSVRG}) is given by
			$$
			K = \mathcal{O}\left(\left(n + \frac{L^2 \zeta}{\mu^2}\right) \log \frac{1}{\varepsilon}\right).
			$$}
		\begin{proof}[Proof of Corollary~\ref{cor:rlsvrg-total-complexity}]
			Choose $B=1.$ Riemannian \algname{L-SVRG} calls the stochastic gradient oracle in expectation $\mathcal{O}(1 + pn)$ times in each iteration. Corollary~\ref{cor:rlsvrg-iteration-complexity} states that the iteration complexity is $\mathcal{O}\left(\left(\frac{1}{p} + \frac{L^2\zeta}{\mu^2}\right)\log\frac{1}{\varepsilon}\right).$ Note that any choice of 
			\[p\in\left[ \min\left\lbrace \frac{\mu^2}{L^2\zeta}, \frac{1}{n} \right\rbrace, \max\left\lbrace \frac{\mu^2}{L^2\zeta}, \frac{1}{n} \right\rbrace \right]\]
			leads to the total complexity of $\mathcal{O}\left(\left(n + \frac{L^2\zeta}{\mu^2}\right)\log\frac{1}{\varepsilon}\right).$ The choice of $p=\frac{1}{n}$ is motivated by the lack of dependence on $L$ and $\mu.$ 
		\end{proof}
		
		\section{R-PAGE: SOTA non-convex Optimization Algorithm}
		\noindent\textbf{Theorem~\ref{thm:page_noncvx}.} \textit{Assuming that each $f_i$ in \eqref{eq:finite-sum} is differentiable and $L$-g-smooth on $\mathcal{M}$ (Assumption~\ref{ass:l-g-smoothness} holds) and that $f$ is lower bounded on $\mathcal{M}$ (Assumption~\ref{ass:lower} holds), let $\delta^0 \stackrel{\text{def}}{=} f(x^0) - f^{*}$. Select the stepsize $\eta$ such that
			$0 < \eta \leq \frac{1}{L\left(1 + \sqrt{\frac{1-p}{pb}}\right)},$ and the minibatch size $B = n,$ and the secondary minibatch size $b < B.$ 
			Then, the iterates of the \algname{R-PAGE} method (Algorithm~\ref{alg:RPAGE}) satisfy
			\begin{equation*}
				\mathbb{E}\left[\norm{\nabla f(\widehat{x}^K)}^2\right] \leq\frac{2\delta^0}{\eta K},
			\end{equation*}
			where $\widehat{x}^k$ is chosen randomly from $x^0,\ldots,x^{K-1}$ with uniform probability distribution.}
		\begin{proof}[Proof of Theorem~\ref{thm:page_noncvx}]
			Let $X = g^k - \nabla f(x^k),$ $a_i = \nabla f_i\left(x^{k+1}\right) - \Gamma_{x^k}^{x^{k+1}}\nabla f_i\left(x^k\right).$ A direct calculation now reveals that
			\begin{equation*}
				\begin{split}
					G &\eqdef \mathbb{E} \left[ \left\| g^{k+1} - \nabla f\left(x^{k+1}\right) \right\|^2\bigg| x^{k+1}, x^k, g^k\right] \\
					& =  p \left\| \nabla f\left(x^{k+1}\right) - \nabla f\left(x^{k+1}\right) \right\|^2\\
					& + (1 - p) \left\| \Gamma_{x^k}^{x^{k+1}}g^k  + \frac{1}{b}\sum_{i\in I}\left( \nabla f_i\left(x^{k+1}\right) - \Gamma_{x^k}^{x^{k+1}}\nabla f_i\left(x^k \right)\right) - \nabla f\left(x^{k+1}\right) \right\|^2 \\
					& = (1 - p) \left\| \Gamma_{x^k}^{x^{k+1}}g^k - \Gamma_{x^k}^{x^{k+1}}\nabla f\left(x^k\right) + \frac{1}{b}\sum_{i\in I}\left( \nabla f_i\left(x^{k+1}\right) - \Gamma_{x^k}^{x^{k+1}}\nabla f_i\left(x^k\right)\right)\right.\\
					& - \left.\left( \nabla f\left(x^{k+1}\right) - \Gamma_{x^k}^{x^{k+1}}\nabla f\left(x^k\right) \right) \right\|^2 \\
					& = (1 - p) \left\| X \right\|^2 + 2(1 - p) \left\langle X, \, \frac{1}{b}\sum_{i\in I}a_i - \bar{a} \right\rangle + (1 - p) \left\| \frac{1}{b}\sum_{i\in I}a_i - \bar{a} \right\|^2,
				\end{split}
			\end{equation*}
			where $\bar{a} = \frac{1}{n} \sum_{i=1}^n a_i.$
			
			Since $\mathbb{E} \left[a_i - \bar{a} \bigg| x^{k+1}, x^k, g^k \right] = 0$ and because $X$ is constant conditioned on $x^{k+1}, x^k, g^k,$ we have
			\begin{eqnarray*}
				\mathbb{E} \left[ G \bigg| x^{k+1}, x^k, g^k \right] &=& (1 - p)\mathbb{E} \left[ \left\| X \right\|^2 \bigg| x^{k+1}, x^k, g^k \right] + (1 - p) \Exp{\left\| a_i - \bar{a} \right\|^2}\\
				&=& (1 - p)\mathbb{E} \left[ \left\| X \right\|^2 \bigg| x^{k+1}, x^k, g^k \right] + (1 - p)\frac{n-b}{(n-1)b} \frac{1}{n} \sum_{i=1}^n \left\| a_i - \bar{a} \right\|^2 \\
				&\leq& (1 - p)\mathbb{E} \left[ \left\| X \right\|^2 \bigg| x^{k+1}, x^k, g^k \right] + \frac{(1 - p)L^2}{b}\left\| \Expmap_{x^{k}}^{-1}(x^{k+1}) \right\|^2, \\
			\end{eqnarray*}
			where in the second step we simply calculate the expectation with respect to the random index, and the last step is due to the smoothness assumption (Assumption~\ref{ass:l-g-smoothness}).
			
			By first applying the three term tower property and subsequently the standard two-term tower property, we get
			\begin{align}\label{eq:rpage-first-recursion}
				\nonumber&\mathbb{E} \left[ \left\| g^{k+1} - \nabla f(x^{k+1}) \right\|^2 \right]\\
				\nonumber& = \mathbb{E} \left[ \mathbb{E} \left[ \mathbb{E} \left[ \left\| g^{k+1} - \nabla f(x^{k+1}) \right\|^2 \bigg| x^{k+1}, x^k, g^k, s^k \right] \bigg| x^{k+1}, x^k, g^k \right] \right]\\
				\nonumber&= \mathbb{E} \left[ \mathbb{E} \left[ G \bigg| x^{k+1}, x^k, g^k \right] \right] \\
				\nonumber&\leq \mathbb{E} \left[ (1 - p)\mathbb{E} \left[ \left\| X \right\|^2 \bigg| x^{k+1}, x^k, g^k \right] + \frac{(1 - p)L^2}{b} \mathbb{E} \left[ \left\| \Expmap_{x^{k}}^{-1}(x^{k+1}) \right\|^2 \right] \right] \\
				\nonumber&= (1 - p)\mathbb{E} \left[ \mathbb{E} \left[ \left\| X \right\|^2 \bigg| x^{k+1}, x^k, g^k \right] \right] + \frac{(1 - p)L^2}{b} \mathbb{E} \left[ \left\| \Expmap_{x^{k}}^{-1}(x^{k+1}) \right\|^2 \right] \\
				\nonumber&= (1 - p)\mathbb{E} \left[ \left\| X \right\|^2 \right] + \frac{(1 - p)L^2}{b} \mathbb{E} \left[ \left\| \Expmap_{x^{k}}^{-1}(x^{k+1}) \right\|^2 \right] \\
				&= (1 - p)\mathbb{E} \left[ \left\| g^k - \nabla f(x^k) \right\|^2 \right] + \frac{(1 - p)L^2}{b} \mathbb{E} \left[ \left\| \Expmap_{x^{k}}^{-1}(x^{k+1}) \right\|^2 \right].
			\end{align}
			
			Let $x \in \mathcal{M},$ $g\in T_x\mathcal{M}.$ By Lemma~\ref{lem:identity_for_dot_prod} with $M=L,$ we have the identity
			\begin{eqnarray}\label{eq:identity_for_dotprod}
				\nonumber\langle \nabla f(x), -\eta g \rangle + \frac{M\eta^2}{2} \|g\|^2 &=& -\frac{\eta}{2} \|\nabla f(x)\|^2 - \left(\frac{1}{2\eta} - \frac{M}{2}\right) \|-\eta g\|^2\\
				&+& \frac{\eta}{2} \|g - \nabla f(x)\|^2.
			\end{eqnarray}
			
			Since $f$ is $L$-g-smooth, we have
			\begin{eqnarray*}
				f(x^{k+1}) & \leq & f(x^{k}) + \langle \nabla f(x^{k}), \Expmap_{x^{k}}^{-1}(x^{k+1}) \rangle + \frac{L}{2}\|\Expmap_{x^{k}}^{-1}(x^{k+1})\|^2] \\
				& \leq & \; f(x^{k})  - \eta \langle\nabla f(x^{k}), g^k\rangle + \frac{L\eta^2}{2} \|g^{k}\|^2.
			\end{eqnarray*}	
			Subtracting $f^{*}$ from both sides, taking expectation and applying~\eqref{eq:identity_for_dotprod}, we get
			\begin{eqnarray}\label{eq:l-g-smooth-inequality}
				\mathbb{E} \left[ f(x^{k+1}) - f^{*} \right] & \leq & \mathbb{E} \left[ f(x^k) - f^{*} \right] - \frac{\eta}{2} \mathbb{E} \left[ \left\| \nabla f(x^k) \right\|^2 \right]\\
				\nonumber & - &\left(\frac{1}{2\eta} - \frac{L}{2} \right) \mathbb{E} \left[ \left\| \Expmap_{x^{k}}^{-1}(x^{k+1}) \right\|^2 \right] +
				\frac{\eta}{2} \Exp{\left\| g^k - \nabla f(x^k) \right\|^2}.
			\end{eqnarray}
			
			Let $\delta^k \stackrel{\text{def}}{=} \mathbb{E} [f(x^k) - f^{*}]$, $s^k \stackrel{\text{def}}{=} \mathbb{E} [\|g^k - \nabla f(x^k)\|^2]$ and $r^k \stackrel{\text{def}}{=} \mathbb{E} [\|\Expmap_{x^{k}}^{-1}(x^{k+1})\|^2]$.
			Then by adding~\eqref{eq:l-g-smooth-inequality} with a $\frac{\eta}{2p}$ multiple of~\eqref{eq:rpage-first-recursion}, we obtain
			\begin{align}\label{eq:rpage-final-recursion}
				\nonumber &\delta^{k+1} + \frac{\eta}{2p} s^{k+1} \\
				\nonumber &\leq \delta^k - \frac{\eta}{2} \Exp{\|\nabla f(x^k)\|^2} - \left(\frac{1}{2\eta} - \frac{L}{2}\right) r^k + \frac{\eta}{2} s^k + \frac{(1 - p)\eta L^2}{2pb}r^k + \frac{\eta}{2p}(1 - p)s^k\\
				\nonumber &	= \delta^k + \frac{\eta}{2p} s^k - \frac{\eta}{2} \Exp{\|\nabla f(x^k)\|^2} - \left(\frac{1}{2\eta} - \frac{L}{2} - \frac{(1 - p)\eta L^2}{2pb}\right) r^k\\
				& \leq \delta^k + \frac{\eta}{2p} s^k - \frac{\eta}{2} \Exp{\|\nabla f(x^k)\|^2}.
			\end{align}
			The last inequality follows from the bound $\frac{(1-p)L^2}{pb}\eta^2 + L\eta \leq 1$, which holds due to the choice of stepsize and Lemma~\ref{lemma:square_iequality}.
			
			By summing up inequalities~\eqref{eq:rpage-final-recursion} for $k = 0, \ldots, K - 1$, we get
			\[
			0 \leq \delta^K + \frac{\eta}{2p} s^K \leq \delta^0 + \frac{\eta}{2p} s^0 - \frac{\eta}{2} \sum_{k=0}^{K-1} \mathbb{E} [\|\nabla f(x^k)\|^2].
			\]	
			Since $g^0 = \nabla f(x^0),$ we have that $s^0 = 0.$ We get that
			\begin{equation*}
				\Exp{\norm{\nabla f(\widehat{x}^K)}^2} = \frac{1}{K}\sum_{k=0}^{K-1}\Exp{\norm{\nabla f(x^k)}^2}\leq\frac{2\delta^0}{\eta K}.
			\end{equation*}
		\end{proof}
		\noindent\textbf{Corollary~\ref{cor:page_noncvx}.} \textit{With the assumptions outlined in Theorem~\ref{thm:page_noncvx}, if we set $\eta = \frac{1}{L\left(1 + \sqrt{\frac{1-p}{pb}}\right)}$, in order to obtain $\mathbb{E}\left[\left\|\nabla f\left(\widehat{x}^K\right)\right\|\right] \leq \varepsilon$, we find that the iteration complexity of Riemannian R-PAGE method (Algorithm~\ref{alg:RPAGE}) is given by
			$$
			K = \mathcal{O}\left(\left(1+\sqrt{\frac{1-p}{pb}}\right) \frac{L\delta^0}{\varepsilon^2}\right).
			$$}
		\begin{proof}[Proof of Corollary~\ref{cor:page_noncvx}]
			Let $\eta  = \frac{1}{L\left(1+\sqrt{\frac{1-p}{pb}}\right)}.$ Then $\Exp{\norm{\nabla f\left(\widehat{x}^K\right)}^2}\leq\varepsilon^2$ as long as $K\geq \frac{2\left(1+\sqrt{\frac{1-p}{pb}}\right)L\delta^0}{\varepsilon^2}.$
		\end{proof}
		\noindent\textbf{Corollary~\ref{cor:page_grad_computes}.} \textit{With the assumptions outlined in Theorem~\ref{thm:page_noncvx}, if we set  $\eta = \frac{1}{L\left(1 + \sqrt{\frac{1-p}{pb}}\right)}$, minibatch size $B=n$, secondary minibatch size $b \leq \sqrt{B}$ and $p \equiv \frac{b}{B+b}$, in order to obtain  $\mathbb{E}\left[\left\|\nabla f\left(\widehat{x}^K\right)\right\|\right] \leq \varepsilon$, we find that the total computational complexity of Riemannian PAGE method (Algorithm~\ref{alg:RPAGE}) is given by
			$$
			\#	\operatorname{grad} =  \mathcal{O}\left(n + \frac{\sqrt{n}}{\varepsilon^2}\right).
			$$}
		\begin{proof}[Proof of Corollary~\ref{cor:page_grad_computes}]
			At each iteration, \algname{R-PAGE} computes $(1-p)b+pn$ new gradients on average. Note that in the preprocessing step $n$ gradients $\nabla f_1(x^0),\ldots,\nabla f_n(x^0)$ are computed. Therefore, the total expected number of gradients computed by \algname{R-PAGE} in order to reach an $\varepsilon$-accurate solution is
			\begin{align*}
				\#	\operatorname{grad} &= n + \left\lceil\frac{2L\delta^0\left(1 + \sqrt{\frac{1-p}{pb}}\right)}{\varepsilon^2} - 1\right\rceil\left((1-p)b+pn\right)\\
				&\leq n + \frac{2L\delta^0\left(1 + \sqrt{\frac{1-p}{pb}}\right)}{\varepsilon^2}\left((1-p)b+pn\right)\\
				& = \mathcal{O}\left(n + \frac{\sqrt{n}}{\varepsilon^2}\right).
			\end{align*}
		\end{proof}
		\section{R-PAGE: SOTA Algorithm Under PL-Condition}
		\noindent\textbf{Theorem~\ref{thm:rpage-pl-condition}.} \textit{    Assuming that each $f_i$ in \eqref{eq:finite-sum} is differentiable and $L$-g-smooth on $\mathcal{M}$ (Assumption~\ref{ass:l-g-smoothness} holds), also that $f$ is lower bounded on $\mathcal{M}$ (Assumption~\ref{ass:lower} holds) and it satisfies Polyak-\L ojasiewicz condition (Assumption~\ref{ass:pl-condition} holds) with constant $\mu$, let $\delta^0 \stackrel{\text{def}}{=} f(x^0) - f^{*}$. Select the stepsize $\eta$ such that $	\eta \leq \min \left\{ \frac{1}{L \left(1 + \sqrt{\frac{1-p}{pb} }\right)}, \frac{p}{2\mu} \right\}.$ Then, the iterates of the Riemannian PAGE method (Algorithm \ref{alg:RPAGE}) satisfy 
			$$\mathbb{E}\left[\Phi^{K}\right] \leq(1-\mu \eta)^K \delta^0,$$
			where $\Phi^k\eqdef f\left(x^k\right)-f^*+\frac{2}{p}\left\|g^k-\nabla f\left(x^k\right)\right\|^2$ and $\Phi^0 = \delta^0$, since $g^0 = \nabla f(x^0).$}
		\begin{proof}[Proof of Theorem~\ref{thm:rpage-pl-condition}.] Add $\beta\times$~\eqref{eq:rpage-first-recursion} to \eqref{eq:l-g-smooth-inequality}, we get
			\begin{align}\label{eq:rpage-lyapunov-recursion-pl}
				\nonumber&\mathbb{E}[\nabla f(x^{k+1}) - f^* + \beta\|g^{k+1} - \nabla f(x^{k+1})\|^2]\\
				\nonumber&\leq \mathbb{E} \left[ (1 - \mu\eta)(f(x^k) - f^*) - \left( \frac{1}{2\eta} - \frac{L}{2}\right) \|\Expmap_{x^{k}}^{-1}(x^{k+1})\|^2 + \frac{\eta}{2} \|g^k - \nabla f(x^k)\|^2 \right]\\
				\nonumber&\quad+ \beta\mathbb{E} \left[ (1 - p)\|g^k - \nabla f(x^k)\|^2 + \frac{(1 - p)L^2}{b} \|\Expmap_{x^{k}}^{-1}(x^{k+1})\|^2 \right]\\
				\nonumber&= \mathbb{E} \left[ (1 - \mu\eta)(f(x^k) - f^*) + \left( \frac{\eta}{2} + (1 - p)\beta \right) \|g^k - \nabla f(x^k)\|^2 \right.\\
				\nonumber&\left.\quad - \left( \frac{1}{2\eta} - \frac{L}{2} - \frac{(1 - p)\beta L^2}{b}\right) \|\Expmap_{x^{k}}^{-1}(x^{k+1})\|^2 \right]\\
				&\leq \mathbb{E} \left[ (1 - \mu\eta)(f(x^k) - f^* + \beta\|g^k - \nabla f(x^k)\|^2) \right],
			\end{align}
			where the last inequality holds by choosing the stepsize
			\[
			\eta \leq \min \left\{\frac{p}{2\mu}, \frac{1}{L \left( 1 + \sqrt{\frac{1 - p}{pb}} \right)} \right\},
			\]
			and $\beta \geq \frac{2}{p}$. Now, we define $\Phi^k := f(x^k) - f^* + \beta\|g^k - \nabla f(x^k)\|^2$, then~\eqref{eq:rpage-lyapunov-recursion-pl} turns to
			\[
			\mathbb{E}[\Phi^{k+1}] \leq (1 - \mu\eta)\mathbb{E}[\Phi^k].
			\]
			Telescoping it for $k = 0,\ldots,K - 1$, we have
			\[
			\mathbb{E}[\Phi^K] \leq (1 - \mu\eta)^K\mathbb{E}[\Phi^0].
			\]
		\end{proof}
		
		\noindent\textbf{Corollary~\ref{cor:rlsvrg-iteration-complexity}.} \textit{With the assumptions outlined in Theorem~\ref{thm:rpage-pl-condition}, if we set $\eta = \min \left\{ \frac{1}{L \left(1 + \sqrt{\frac{1-p}{pb} }\right)}, \frac{p}{2\mu} \right\}$, in order to obtain $\mathbb{E}[\Phi^k] \leq \varepsilon \Phi^0$, we find that the iteration complexity of Riemannian R-PAGE method (Algorithm~\ref{alg:RPAGE}) is given by
			$$
			K =\mathcal{O}\left( \left( \left(1 + \sqrt{\frac{1-p}{pb}}\right) \frac{L}{\mu} + \frac{1}{p} \right) \log \frac{1}{\varepsilon}\right).
			$$}
		\begin{proof}[Proof of Corollary~\ref{cor:rlsvrg-iteration-complexity}]
			Note that $\Phi^0 = f(x^0) - f^* + \beta\|g^0 - \nabla f(x^0)\|^2 = f(x^0) - f^* \eqdef \delta^0$, we have
			\[
			\mathbb{E}[f(x^K) - f^*] \leq (1 - \mu\eta)^K\delta^0
			= \varepsilon,
			\]
			where the last equality holds due to the choice of the number of iterations
			\[
			K = \frac{1}{\mu\eta} \log \frac{\delta^0}{\varepsilon} = \left( \left( 1 + \sqrt{\frac{1 - p}{pb}} \right) \kappa + \frac{2}{p} \right) \log \frac{\delta^0}{\varepsilon},
			\]
			where $\kappa := \frac{L}{\mu}$.
		\end{proof}
		
		Now, we restate the corollary, in which a detailed convergence result is obtained by giving a specific parameter setting, and then provide its proof.
		
		\noindent\textbf{Corollary~\ref{cor:rpage-pl-total-cxty}.} \textit{Suppose that each $f_i$ satisfies Assumption~\ref{ass:l-g-smoothness} with constant $L\geq 0$ and Assumption~\ref{ass:pl-condition} holds for $f$ with constant $\mu > 0.$ Let stepsize
			\[
			\eta \leq \min \left\{ \frac{1}{L\left(1 + \frac{\sqrt{n}}{b}\right)}, \frac{b}{2\mu(n + b)} \right\},
			\]
			and probability $p = \frac{b}{n+b}.$ Then the number of iterations performed by \algname{R-PAGE} to find an $\varepsilon$-solution of non-convex finite-sum problem~\eqref{eq:finite-sum} can be bound by $K = \left( \left(1 + \frac{\sqrt{n}}{b}\right) \kappa + \frac{2(n+b)}{b} \right) \log \frac{\delta^0}{\varepsilon}.$ Moreover, the number of stochastic gradient computations (gradient complexity) is
			\[
			\#grad \leq n + \left(4\sqrt{n}\kappa + 4n\right)\log \frac{\delta^0}{\varepsilon} = \mathcal{O} \left( (n + \sqrt{n}\kappa) \log \frac{1}{\varepsilon} \right).
			\]
		}
		\begin{proof}[Proof of Corollary~\ref{cor:rpage-pl-total-cxty}.]
			If we choose probability $p = \frac{b}{n+b}$, then the term $ \sqrt{\frac{1 - p}{pb}}$ becomes $\frac{\sqrt{n}}{b}.$ Thus, according to Theorem~\ref{thm:rpage-pl-condition}, the stepsize is
			\[
			\eta \leq \min \left\{ \frac{1}{L\left(1 + \frac{\sqrt{n}}{b}\right)}, \frac{b}{2\mu(n + b)} \right\},
			\]
			and the total number of iterations $K = \left( \left(1 + \frac{\sqrt{n}}{b}\right) \kappa + \frac{2(n+b)}{b} \right) \log \frac{\delta^0}{\varepsilon}.$
			
			According to the gradient estimator of \algname{R-PAGE} (Line~\ref{line:rpage-estimator} of Algorithm~\ref{alg:RPAGE}), we know that it uses $pn + (1 - p)b = \frac{2nb}{n + b}$ stochastic gradients for each iteration on the expectation. Thus, the gradient complexity
			\begin{align*}
				\#grad 
				&= n + K(pn + (1 - p)b)\\
				& = n + \frac{2n}{n + 1}\left(\left(1 + \frac{\sqrt{n}}{b}\right)\kappa + \frac{2(n + b)}{b}\right)\log \frac{\delta^0}{\varepsilon}\\
				& = n + \left(\frac{2nb}{n + b}\left(1 + \frac{\sqrt{n}}{b}\right)\kappa + 4n\right)\log \frac{\delta^0}{\varepsilon}\\
				& \leq n + \left(2b\left(1 + \frac{\sqrt{n}}{B}\right)\kappa + 4n\right)\log \frac{\delta^0}{\varepsilon}\\
				& \leq n + \left(4\sqrt{n}\kappa + 4n\right)\log \frac{\delta^0}{\varepsilon}\\
				& = \mathcal{O} \left( \left(n + \sqrt{n}\kappa\right) \log \frac{1}{\varepsilon} \right).
			\end{align*}
		\end{proof}
		\section{R-PAGE: Online non-convex Case}
		In this appendix, we provide the detailed proofs for our main convergence theorem and its corollaries for \algname{R-PAGE} in the non-convex online case (i.e., problem~\eqref{eq:online-problem}). 
		
		We first restate the main convergence result (Theorem~\ref{thm:rpage-noncvx-online}) in the non-convex online case and then provide its proof.
		
		\noindent\textbf{Theorem~\ref{thm:rpage-noncvx-online}.} \textit{Suppose that Assumption~\ref{ass:bounded-variance} and Assumption~\ref{ass:unbiasedness} hold for the stochastic gradient $\nabla f_{B}(x)$ and suppose that, for any fixed $\xi_i,$ $f(x,\xi_i)$ is geodesically $L$-smooth on $\mathcal{M}$ (Assumption \ref{ass:l-g-smoothness} holds), also that $f$ is lower bounded on $\mathcal{M}$ (Assumption~\ref{ass:lower} holds). Let $\delta^0 \stackrel{\text{def}}{=} f(x^0) - f^{*}$. Select the stepsize $\eta$ such that $		\eta \leq \frac{1}{L \left( 1 + \sqrt{\frac{1-p}{pb}} \right)}.$ Then, the iterates of the Riemannian PAGE method (Algorithm \ref{alg:RPAGE}) satisfy 
			$$
			\mathbb{E}\|\nabla f(\widehat{x}^K)\|^2 \leq \frac{2\mathbb{E}[\Phi^0]}{\eta K} + \frac{\sigma^2}{B}.
			$$
			where $\widehat{x}^{K}$ is chosen uniformly at random from $x^0, \dots, x^{K-1}$ and $ \Phi^0 = f\left(x^0\right)-f^*+\frac{\eta \sigma^2}{2 p B}$}
		\begin{proof}[Proof of Theorem~\ref{thm:rpage-noncvx-online}]
			According to the update step $ x^{k+1} = \Expmap_{x^k}\left(-\eta g^k\right)$ (see Line~4 in Algorithm~\ref{alg:RPAGE}) and~\eqref{eq:l-g-smooth-inequality}, we have
			\begin{equation}\label{eq:rpage-online-recursion-first}
				f(x^{k+1}) \leq f(x^k) - \frac{\eta}{2} \|\nabla f(x^k)\|^2 - \left( \frac{1}{2\eta} - \frac{L}{2} \right) \left\|\Expmap_{x^{k}}^{-1}\left(x^{k+1}\right)\right\|^2 + \frac{\eta}{2} \|g^k - \nabla f(x^k)\|^2.
			\end{equation}
			Now, we use the following Lemma~\ref{lemma:recusion-online} to bound the last term.
			
			\begin{lemma}\label{lemma:recusion-online}
				Suppose that Assumption~\ref{ass:bounded-variance} holds for the stochastic gradient $\nabla f_{B}(x)$ and suppose that, for any fixed $\xi_i,$ $f(x,\xi_i)$ is geodesically $L$-smooth in $x\in\mathcal{M}.$ If the gradient estimator $g^{k+1}$ is defined in Line~\ref{line:rpage-estimator} of Algorithm~\ref{alg:RPAGE}, then we have
				\begin{equation}\label{eq:lemma-recusion-online}
					\mathbb{E}\left[\|g^{k+1} - \nabla f(x^{k+1})\|^2\right] \leq (1 - p)\|g^k - \nabla f(x^k)\|^2 + \frac{(1 - p)L^2}{b}\left\|\Expmap_{x^{k}}^{-1}(x^{k+1})\right\|^2 + \frac{p\sigma^2}{B}.
				\end{equation}
			\end{lemma}
			\begin{proof}[Proof of Lemma~\ref{lemma:recusion-online}]
				According to the definition of \algname{R-PAGE} gradient estimator in Line~\ref{line:rpage-estimator} of Algorithm~\ref{alg:RPAGE}
				\[
				g^{k+1} =
				\begin{cases}
					\nabla f_{B}(x^{k+1}) & \text{with probability } p, \\
					\Gamma_{x^k}^{x^{k+1}}g^k + \nabla f_{b}(x^{k+1}) - \Gamma_{x^k}^{x^{k+1}}\nabla f_{b}(x^{k}) & \text{with probability } 1 - p,
				\end{cases}
				\]
				we have
				\begin{align}
					\nonumber&\Exp{\left\|g^{k+1} - \nabla f(x^{k+1})\right\|^2}\\
					\nonumber& = p\Exp{\left\|\nabla f_{B}(x^{k+1}) - \nabla f(x^{k+1})\right\|^2}\\
					\nonumber& + (1 - p)\Exp{\left\|\Gamma_{x^k}^{x^{k+1}}g^k + \nabla f_{b}(x^{k+1}) - \Gamma_{x^k}^{x^{k+1}}\nabla f_{b}(x^{k}) - \nabla f(x^{k+1})\right\|^2}\\
					\nonumber& \stackrel{\text{Asm.~}\ref{ass:bounded-variance}}{\leq} \frac{p\sigma^2}{B} + (1 - p)\Exp{\left\|\Gamma_{x^k}^{x^{k+1}}g^k + \nabla f_{b}(x^{k+1}) - \Gamma_{x^k}^{x^{k+1}}\nabla f_{b}(x^{k})
						- \nabla f(x^{k+1})\right\|^2}\\
					\nonumber& = \frac{p\sigma^2}{B}\\
					\nonumber& + (1 - p)\Exp{\left\|\Gamma_{x^k}^{x^{k+1}}g^k - \Gamma_{x^k}^{x^{k+1}}\nabla f(x^k) + \nabla f_{b}(x^{k+1}) - \Gamma_{x^k}^{x^{k+1}}\nabla f_{b}(x^{k}) - \nabla f(x^{k+1}) + \Gamma_{x^k}^{x^{k+1}}\nabla f(x^k)\right\|^2}\\
					\nonumber& = \frac{p\sigma^2}{B} + (1 - p)\left\|g^k - \nabla f(x^k)\right\|^2\\
					\nonumber& + (1 - p)\Exp{\left\|\nabla f_{b}(x^{k+1}) - \Gamma_{x^k}^{x^{k+1}}\nabla f_{b}(x^{k}) - (\nabla f(x^{k+1}) - \Gamma_{x^k}^{x^{k+1}}\nabla f(x^k))\right\|^2}\\
					\nonumber& = \frac{p\sigma^2}{B} + (1 - p)\|g^k - \nabla f(x^k)\|^2\\
					\nonumber& + (1 - p)\Exp{\left\|\nabla f_{b}(x^{k+1}) - \Gamma_{x^k}^{x^{k+1}}\nabla f_{b}(x^{k}) - (\nabla f(x^{k+1}) - \Gamma_{x^k}^{x^{k+1}}\nabla f(x^k))\right\|^2}\\
					\nonumber& = \frac{p\sigma^2}{B} + (1 - p)\|g^k - \nabla f(x^k)\|^2\\
					\nonumber& + \frac{1-p}{b^2} \Exp{\left\| \sum_{i \in I'}\left(\nabla f(x^{k+1},\xi_i) - \Gamma_{x^k}^{x^{k+1}}\nabla f(x^k,\xi_i)\right) - \left(\nabla f(x^{k+1}) - \Gamma_{x^k}^{x^{k+1}}\nabla f(x^k)\right)\right\|^2}\\
					\nonumber& \leq \frac{p\sigma^2}{B} + (1 - p)\|g^k - \nabla f(x^k)\|^2 + \frac{1 - p}{b} \Exp{\left\|\nabla f(x^{k+1}, \xi_i) - \Gamma_{x^k}^{x^{k+1}}\nabla f(x^k, \xi_i)\right\|^2}\\
					\nonumber&\stackrel{\text{Asm.~}\ref{ass:l-g-smoothness}}{\leq} \frac{p\sigma^2}{B} + (1 - p)\|g^k - \nabla f(x^k)\|^2 + \frac{(1 - p)L^2}{b}\left\|\Expmap_{x^{k}}^{-1}(x^{k+1})\right\|^2.
				\end{align}
			\end{proof}
			
			Now, we continue to prove Theorem using Lemma~\ref{lemma:recusion-online}. We add \eqref{eq:rpage-online-recursion-first} to $ \frac{\eta}{2p} \times \eqref{eq:lemma-recusion-online} $, and take expectation to get
			\begin{align}
				\nonumber&\mathbb{E}\left[f(x^{k+1}) - f^* + \frac{\eta}{2p}\left\|g^{k+1} - \nabla f(x^{k+1})\right\|^2\right]\\
				\nonumber&\leq \mathbb{E}\left[f(x^k) - f^* - \frac{\eta}{2}\|\nabla f(x^k)\|^2 - \left(\frac{1}{2\eta} - \frac{L}{2}\right)\|\Expmap_{x^{k}}^{-1}(x^{k+1})\|^2 + \frac{\eta}{2}\|g^k - \nabla f(x^k)\|^2\right]\\
				\nonumber& + \frac{\eta}{2p}\mathbb{E}\left[(1 - p)\|g^k - \nabla f(x^k)\|^2 + \frac{(1 - p)L^2}{b}\|\Expmap_{x^{k}}^{-1}(x^{k+1})\|^2 + \frac{p\sigma^2}{B}\right]\\
				\nonumber&= \mathbb{E}\left[f(x^k) - f^* + \frac{\eta}{2p}\|g^k - \nabla f(x^k)\|^2 - \frac{\eta}{2}\|\nabla f(x^k)\|^2 + \frac{\eta\sigma^2}{2B}\right.\\
				\nonumber&\left.- \left(\frac{1}{2\eta} - \frac{L}{2} - \frac{(1 - p)\eta L^2}{2pb}\right)\|\Expmap_{x^{k}}^{-1}(x^{k+1})\|^2\right]\\
				\nonumber&\leq \mathbb{E}\left[f(x^k) - f^* + \frac{\eta}{2p}\|g^k - \nabla f(x^k)\|^2 - \frac{\eta}{2}\|\nabla f(x^k)\|^2 + \frac{\eta\sigma^2}{2B}\right],
			\end{align}
			where the last inequality holds due to $\frac{1}{2\eta} - \frac{L}{2} - \frac{(1-p)\eta L^2}{2pb} \geq 0$ by choosing stepsize
			\[
			\eta \leq \frac{1}{L\left(1 + \sqrt{\frac{1-p}{pb}}\right)}.
			\]
			Now, if we define $\Phi^k \eqdef f(x^k) - f^* + \frac{\eta}{2p}\|g^k - \nabla f(x^k)\|^2$, then turns to
			\[
			\mathbb{E}[\Phi^{k+1}] \leq \mathbb{E}[\Phi^k] - \frac{\eta}{2}\Exp{\|\nabla f(x^k)\|^2} + \frac{\eta\sigma^2}{2B}.
			\]
			Summing up it from $ k = 0 $ for $ K - 1 $, we have
			\[
			\mathbb{E}[\Phi^K] \leq \mathbb{E}[\Phi^0] - \frac{\eta}{2} \sum_{k=0}^{K-1}\mathbb{E}\|\nabla f(x^k)\|^2 + \frac{\eta K\sigma^2}{2B}.
			\]
			Then, according to the output of \algname{R-PAGE}, i.e., $\widehat{x}^K$ is randomly chosen from $ \{x^k\}_{k=0}^{K-1} $, we have
			\begin{equation}\label{eq:rpage-online-prefinal-recursion}
				\mathbb{E}\|\nabla f(\widehat{x}^K)\|^2 \leq \frac{2\mathbb{E}[\Phi^0]}{\eta K} + \frac{\sigma^2}{B}.
			\end{equation}
		\end{proof}
		\noindent\textbf{Corollary~\ref{cor:rpage-iteration-cxty-online}}
		\textit{With the assumptions outlined in Theorem~\ref{thm:rpage-noncvx-online}, if we set $\eta =  \frac{1}{L \left(1 + \sqrt{\frac{1-p}{pb} }\right)}$, in order to obtain $\mathbb{E}[\Phi^k] \leq \varepsilon \Phi^0$, we find that the iteration complexity of Riemannian R-PAGE method (Algorithm~\ref{alg:RPAGE}) is given by
			$$
			K =\mathcal{O}\left(  \frac{\delta^0L}{\varepsilon^2} \left( 1 + \sqrt{\frac{1-p}{pb}} \right)  + \frac{1}{p}\right).
			$$}
		\begin{proof}[Proof of Corollary~\ref{cor:rpage-iteration-cxty-online}]
			For the term $ \mathbb{E}[\Phi^0] $, we have
			\begin{align*}
				\Exp{\Phi^0}
				&= \Exp{f(x^0) - f^* + \frac{\eta}{2p}\|g^0 - \nabla f(x^0)\|^2}\\
				&= \Exp{f(x^0) - f^* + \frac{\eta}{2p} \left\|  \nabla f_{B}(x^0) - \nabla f(x^0) \right\|^2}\\
				&\stackrel{\text{Asm.~}\ref{ass:bounded-variance}}{\leq} f(x^0) - f^* + \frac{\eta\sigma^2}{2pB},
			\end{align*}
			Plugging the last bound on $\Exp{\Phi^0}$ into \eqref{eq:rpage-online-prefinal-recursion} and noting that $ \delta^0 \eqdef f(x^0) - f^*,$ we obtain
			\begin{align*}
				\Exp{\|\nabla f\left(\widehat{x}^K\right)\|^2}
				&\leq \frac{2\delta^0}{\eta K} + \frac{\sigma^2}{pB K} + \frac{\sigma^2}{B}\\
				&\leq \frac{2\delta^0}{\eta K} + \frac{\varepsilon^2}{2pK} + \frac{\varepsilon^2}{2}\\
				&= \varepsilon^2,
			\end{align*}
			where the last equality holds by letting the number of iterations
			\[
			K = \frac{4\delta^0}{\varepsilon^2\eta} + \frac{1}{p} = \frac{4\delta^0L}{\varepsilon^2} \left( 1 + \sqrt{\frac{1-p}{pb}} \right) + \frac{1}{p}.
			\]
		\end{proof}
		
		\noindent\textbf{Corollary~\ref{cor:rpage-noncvx-online-optimality}.} \textit{With the assumptions outlined in Theorem~\ref{thm:rpage-noncvx-online}, if we set $\eta =  \frac{1}{L \left(1 + \sqrt{\frac{1-p}{pb} }\right)}$, 	minibatch size $B = \left\lceil\frac{2\sigma^2}{\varepsilon^2}\right\rceil,$ secondary minibatch size $b \leq \sqrt{B}$ and probability $p = \frac{b}{B+b},$ in order to obtain $\mathbb{E}[\Phi^k] \leq \varepsilon \Phi^0$, we find that the total computational complexity of Riemannian PAGE method of Riemannian R-PAGE method (Algorithm~\ref{alg:RPAGE}) is given by
			$$
			\#grad =  \mathcal{O}\left(\frac{\sigma^2}{\varepsilon^2} + \frac{\sigma}{\varepsilon^3}\right).
			$$}
		\begin{proof}[Proof of Corollary~\ref{cor:rpage-noncvx-online-optimality}] 
			If we choose probability $p = \frac{b}{B+b},$ then $\sqrt{\frac{1-p}{pb}} = \frac{\sqrt{B}}{b}.$ Thus, according to Theorem~\ref{thm:rpage-noncvx-online}, the stepsize bound becomes $\eta \leq \frac{1}{L\left(1+\frac{\sqrt{B}}{b}\right)}$ and the total number of iterations becomes $K = \frac{4\delta^0L}{\varepsilon^2}\left(1 + \frac{\sqrt{B}}{b}\right) + \frac{B+b}{b}.$ Since the gradient estimator of \algname{R-PAGE} (Line~\ref{line:rpage-estimator} of Algorithm~\ref{alg:RPAGE}) uses $pB + (1 - p)b = \frac{2B b}{B+b}$ stochastic gradients in each iteration in expectation, the gradient complexity is
			\begin{align*}
				\#grad 
				&= B + K (pB + (1 - p)b)\\
				&= B + \left(\frac{4\delta^0L}{\varepsilon^2} \left(1 + \frac{\sqrt{B}}{b}\right) + \frac{B+b}{b}\right) \frac{2B b}{B+b}\\
				&= 3B + \frac{4\delta^0L}{\varepsilon^2} \left(1 + \frac{\sqrt{B}}{b}\right) \frac{2B b}{B+b}\\
				&\leq 3B + \frac{4\delta^0L}{\varepsilon^2} \left(1 + \frac{\sqrt{B}}{b}\right) 2b\\
				&\leq 3B + \frac{16\delta^0L\sqrt{B}}{\varepsilon^2}\\
				& =  \mathcal{O}\left(\frac{\sigma^2}{\varepsilon^2} + \frac{\sigma}{\varepsilon^3}\right).
			\end{align*}
			where the last inequality is due to the parameter setting $b \leq \sqrt{B}.$
		\end{proof}
		\section{R-PAGE: Online Case P\L-Condition}
		\begin{theorem}\label{thm:rpage-online-pl}
			Suppose that Assumptions~\ref{ass:l-g-smoothness},~\ref{ass:pl-condition}~and~\ref{ass:bounded-variance} hold on $\mathcal{M}.$ Choose the stepsize
			\[
			\eta \leq \min \left\{ \frac{1}{L \left( 1 + \sqrt{\frac{1-p}{pb}} \right)}, \frac{p}{2\mu} \right\},
			\]
			minibatch size $ B = \frac{2\sigma^2}{\eta^2}$, secondary minibatch size $ b \leq B $, and probability $p \in (0, 1].$ For $k=0,\ldots,K-1,$ let $ \Phi^k := f(x^k) - f^* + \frac{\eta}{p} \|g^k - \nabla f(x^k)\|^2.$ Then $\mathbb{E}[\Phi^K] \leq (1 - \mu\eta)^K\mathbb{E}[\Phi^0] +  \frac{\sigma^2}{B\mu}.$
		\end{theorem}
		\begin{proof}[Proof of Theorem~\ref{thm:rpage-online-pl}.]
			According to Lemma~\ref{lem:identity_for_dot_prod} and Lemma~\ref{lemma:recusion-online}, we add~\eqref{eq:rpage-online-recursion-first} to $ \beta \times~\eqref{eq:lemma-recusion-online},$ and take expectation to get
			\[
			\mathbb{E} \left[ f(x^{k+1}) - f^* + \beta\|g^{k+1} - \nabla f(x^{k+1})\|^2 \right]
			\]
			\[
			\leq \mathbb{E} \left[ (1 - \mu\eta)(f(x^k) - f^*) - \left( \frac{1}{2\eta} - \frac{L}{2} \right) \|\Expmap_{x^{k}}^{-1}(x^{k+1})\|^2 + \frac{\eta}{2} \|g^k - \nabla f(x^k)\|^2 \right]
			\]
			\[
			+ \beta \mathbb{E} \left[ (1 - p)\|g^k - \nabla f(x^k)\|^2 + \frac{(1 - p)L^2}{b} \|\Expmap_{x^{k}}^{-1}(x^{k+1})\|^2 +  \frac{p\sigma^2}{B} \right]
			\]
			\[
			= \mathbb{E} \left[ (1 - \mu\eta)(f(x^k) - f^*) + \left(\frac{\eta}{2} + (1 - p)\beta \right) \|g^k - \nabla f(x^k)\|^2 +  \frac{\beta p\sigma^2}{2B} \right]
			\]
			\[
			- \left( \frac{1}{2\eta} - \frac{L}{2} - \frac{(1 - p)\beta L^2}{2b} \right) \|\Expmap_{x^{k}}^{-1}(x^{k+1})\|^2
			\]
			\[
			\leq \mathbb{E} \left[ (1 - \mu\eta) \left( f(x^k) - f^* + \frac{\beta}{2p} \|g^k - \nabla f(x^k)\|^2 \right) +  \frac{\beta p\sigma^2}{B} \right],
			\]
			where the last inequality holds due to the choice of the stepsize
			\[
			\eta \leq \min \left\{\frac{p}{2\mu}, \frac{1}{L \left( 1 + \sqrt{\frac{1-p}{pb}} \right)} \right\},
			\]
			and $ \beta \geq \frac{\eta}{p} $. Now, we define $ \Phi^k := f(x^k) - f^* + \beta \|g^k - \nabla f(x^k)\|^2 $ and choose $ \beta = \frac{\eta}{p} $, then we obtain that
			\[
			\mathbb{E}[\Phi^{k+1}] \leq (1 - \mu\eta)\mathbb{E}[\Phi^k] +  \frac{\eta\sigma^2}{B}.
			\]
			Telescoping it from $k = 0,\ldots,K - 1$, we have
			\[
			\mathbb{E}[\Phi^K] \leq (1 - \mu\eta)^K\mathbb{E}[\Phi^0] +  \frac{\sigma^2}{B\mu}.
			\]
		\end{proof}
		\begin{corollary}\label{cor:rpage-online-pl-iteration-cxty}
			Under the assumptions of Theorem~\ref{thm:rpage-online-pl}, the number of iterations performed by \algname{R-PAGE} sufficient for finding an $\varepsilon$-accurate solution of non-convex online problem~\eqref{eq:online-problem} can be bounded by
			\[
			K = \left( \left(1 + \sqrt{\frac{1-p}{pb}}\right) \kappa + \frac{2}{p} \right) \log \frac{2\delta^0}{\varepsilon},
			\]
			where $ \kappa := \frac{L}{\mu}.$
		\end{corollary}
		\begin{proof} Letting the minibatch size $ b = \left\lceil\frac{2\sigma^2}{\mu \varepsilon^2}\right\rceil$ and the number of iterations
			\[
			K = \frac{1}{\mu\eta} \log \frac{2\delta^0}{\varepsilon} = \left( \left(1 + \sqrt{\frac{1-p}{pb}}\right) \kappa + \frac{2}{p} \right) \log \frac{2\delta^0}{\varepsilon},
			\]
			we obtain that
			\[
			\mathbb{E}[\Phi^K] \leq (1 - \mu\eta)^K\mathbb{E}[\Phi^0] +  \frac{\sigma^2}{B\mu} = \frac{\varepsilon}{2} + \frac{\varepsilon}{2} = \varepsilon.
			\]
		\end{proof}
		\begin{corollary}\label{cor:rpage-online-pl-gradient-cxty}
			Under the assumptions of Theorem~\ref{thm:rpage-online-pl}, the number of stochastic gradient computations (i.e., gradient complexity) is
			\[
			\#grad = B + K(pB + (1 - p)b) = B + (pB + (1 - p)b) \left( \left(1 + \sqrt{\frac{1-p}{pb}}\right) \kappa + \frac{2}{p} \right) \log \frac{2\delta^0}{\varepsilon}.
			\]
		\end{corollary}
		\begin{proof}[Proof of Corollary~\ref{cor:rpage-online-pl-gradient-cxty}]
			It follows from the fact that \algname{R-PAGE} uses $B$ gradients for the computation of $g^0$ (see Line~\ref{line:rpage-0-estimator} of Algorithm~\ref{alg:RPAGE}) and $ pB + (1 - p)b = \frac{2B b}{B+b} $ stochastic gradients for each iteration on the expectation.
		\end{proof}
		
		Now let us restate the corollaries under a specific parameter setting. We obtain more detailed convergence results.
		\begin{corollary}\label{cor:rpage-online-pl}
			Suppose that Assumptions~\ref{ass:l-g-smoothness},~\ref{ass:pl-condition}~and~\ref{ass:bounded-variance} hold on $\mathcal{M}.$ Choose the stepsize
			\[
			\eta \leq \min\left\{\frac{1}{L(1+\sqrt{B}/b)}, \frac{b}{2\mu(B+b)}\right\},
			\]
			minibatch size $B = \left\lceil\frac{2\sigma^2}{\mu \varepsilon^2}\right\rceil,$ secondary minibatch size $b \leq \sqrt{B}$
			and probability $ p = \frac{b}{B+b}.$ Then the number of iterations performed by \algname{R-PAGE} to find an $\varepsilon$-solution of non-convex online problem~\eqref{eq:online-problem} can be bounded by
			\[
			K = \left( \left(1 + \frac{\sqrt{B}}{b}\right) \kappa + \frac{2\left(b+B\right)}{B} \right) \log \frac{2\delta^0}{\varepsilon}.
			\]
			Moreover, the number of stochastic gradient computations (gradient complexity) is
			\[
			\#grad = \mathcal{O} \left( B + \sqrt{B}\kappa \log \frac{1}{\varepsilon} \right).
			\]
		\end{corollary}
		\begin{proof}[Proof of Corollary~\ref{cor:rpage-online-pl}.]
			If we choose probability $ p = \frac{b}{B+b} $, then this term $ \sqrt{\frac{1-p}{pb}} = \frac{\sqrt{B}}{b} $. Thus, according to Theorem~\ref{thm:rpage-online-pl}, the stepsize $ \eta $ satisfies
			\[
			\eta \leq \min\left\{\frac{1}{L(1+\sqrt{B}/b)}, \frac{b}{2\mu(B+b)}\right\}
			\]
			and the total number of iterations $K$ is
			\[
			K = \left( \left(1 + \frac{\sqrt{B}}{b}\right) \kappa + \frac{2\left(B+b\right)}{b} \right) \log \frac{2\delta^0}{\varepsilon}.
			\]
			According to the gradient estimator of \algname{R-PAGE} (Line~\ref{line:rpage-estimator} of Algorithm~\ref{alg:RPAGE}), we know that it uses $ pB + (1 - p)b = \frac{2B b}{B+b} $ stochastic gradients for each iteration on the expectation. Thus, the gradient complexity
			\begin{align}
				\#grad 
				\nonumber &= B + K(pB + (1 - p)b)\\
				\nonumber &= B + \frac{2B b}{B + b} \left( \left(1 + \frac{\sqrt{B}}{b}\right) \kappa + \frac{2\left(B + b\right)}{b} \right) \log \frac{2\delta^0}{\varepsilon}\\
				\nonumber &= B + \left(\frac{2B b}{B + b} \left(1 + \frac{\sqrt{B}}{b}\right) \kappa + 4B\right) \log \frac{2\delta^0}{\varepsilon}\\
				\nonumber &\leq B + \left(2b\left(1 + \frac{\sqrt{B}}{b}\right) \kappa + 4B\right) \log \frac{2\delta^0}{\varepsilon}\\
				\nonumber &\leq B + \left(4\sqrt{B}\kappa + 4B\right) \log \frac{2\delta^0}{\varepsilon},
			\end{align}
			where the last inequality is due to the parameter setting $ b \leq \sqrt{B}.$
		\end{proof}
		
		\section{R-MARINA: SOTA Distributed Optimization Algorithm}
		In this section, we provide the statement of Theorem~\ref{thm:marina_noncvx} together with the proof of this result.
		
		\noindent\textbf{Theorem~\ref{thm:marina_noncvx}.} \textit{Assume each $f_i$ is $L$-g-smooth on $\mathcal{M}$ (Assumption~\ref{ass:l-g-smoothness} holds) and let $f$ be uniformly lower bounded on $\mathcal{M}$ (Assumption~\ref{ass:lower} holds). Assume that the compression operator is unbiased and has conic variance (Definition \ref{def:quantization}). Let $\delta^0 \stackrel{\text{def}}{=} f(x^0) - f^{*}$. Select the stepsize $\eta$ such that $		\eta \leq \frac{1}{L \left( 1 + \sqrt{\frac{1-p}{p}\frac{\omega}{n}} \right)}.$ Then, the iterates of the Riemannian MARINA method (Algorithm \ref{alg:R-MARINA}) satisfy 
			\begin{equation}\label{eq:rmarina_convergence_guarantees}
				\mathbb{E}\left[\|f\left(\widehat{x}^{K}\right)\|^2\right] \leq\frac{2 \delta^0}{\gamma K},
			\end{equation}
			where $\widehat{x}^{K}$ is chosen uniformly at random from $x^0, \dots, x^{K-1}$.}
		\begin{proof}[Proof of Theorem~\ref{thm:marina_noncvx}]
			The scheme of the proof is similar to the proof of Theorem~2.1 from~\citep{gorbunov2021marina}. From~\eqref{eq:l-g-smooth-inequality}, we have
			\begin{eqnarray*}
				\mathbb{E} \left[ f(x^{k+1}) - f^{*} \right] & \leq & \mathbb{E} \left[ f(x^k) - f^{*} \right] - \frac{\eta}{2} \mathbb{E} \left[ \left\| \nabla f(x^k) \right\|^2 \right]\\
				& - &\left(\frac{1}{2\eta} - \frac{L}{2} \right) \mathbb{E} \left[ \left\| \Expmap_{x^{k}}^{-1}(x^{k+1}) \right\|^2 \right] +
				\frac{\eta}{2} \Exp{\left\| g^k - \nabla f(x^k) \right\|^2}.\\
			\end{eqnarray*}
			Next, we need to derive an upper bound for $\Exp{\norm{g^{k+1} - \nabla f(x^{k+1})}^2}$. By definition of $g^{k+1}$, we have
			\begin{equation*}
				g^{k+1} =
				\begin{cases}
					\nabla f(x^{k+1}) & \text{with probability } p, \\
					\Gamma_{x^k}^{x^{k+1}}g^k + \frac{1}{n}\sum_{i=1}^n Q(\nabla f_i(x^{k+1}) - \Gamma_{x^k}^{x^{k+1}}\nabla f_i(x^k)) & \text{with probability } 1-p.
				\end{cases}
			\end{equation*}
			Using this, variance decomposition \eqref{eq:variance-decomposition} and tower property~\eqref{eq:tower-property}, we derive:
			\begin{align*}
				&\Exp{ \norm{g^{k+1} - \nabla f(x^{k+1})}^2}\\
				&\stackrel{\eqref{eq:tower-property}}{=}
				(1 - p)\Exp{ \norm{\Gamma_{x^k}^{x^{k+1}}g^k + \frac{1}{n}\sum_{i=1}^n Q(\nabla f_i\left(x^{k+1}\right) - \Gamma_{x^k}^{x^{k+1}}\nabla f_i(x^k)) - \nabla f\left(x^{k+1}\right)}^2}\\
				&\stackrel{\eqref{eq:tower-property},\eqref{eq:variance-decomposition}}{=}
				(1 - p)\Exp{ \left\| \frac{1}{n}\sum_{i=1}^n Q(\nabla f_i(x^{k+1}) - \Gamma_{x^k}^{x^{k+1}}\nabla f_i(x^k)) - \nabla f(x^{k+1}) + \Gamma_{x^k}^{x^{k+1}}\nabla f(x^k) \right\|^2}\\
				&+ (1 - p)\Exp{ \|g^k - \nabla f(x^k)\|^2}.
			\end{align*}
			Since $Q\left(\nabla f_i(x^{k+1}) - \Gamma_{x^k}^{x^{k+1}}\nabla f_i(x^k)\right), \ldots,  Q\left(\nabla f_n(x^{k+1}) - \Gamma_{x^k}^{x^{k+1}}\nabla f_n(x^k)\right)$ are independent random vectors for fixed $x^k$ and $x^{k+1}$ we have
			\begin{align*}
				&\Exp{\norm{g^{k+1} - \nabla f(x^{k+1})}^2}\\
				& = (1 - p)\Exp{\norm{\frac{1}{n} \sum_{i=1}^n Q\left(\nabla f_i\left(x^{k+1}\right) - \Gamma_{x^k}^{x^{k+1}}\nabla f_i(x^k)\right) - \nabla f\left(x^{k+1}\right) + \Gamma_{x^k}^{x^{k+1}}\nabla f(x^k)}^2}\\
				& + (1 - p)\Exp{\norm{g^k - \nabla f(x^k)}^2}\\
				& = \frac{1 - p}{n^2} \sum_{i=1}^n \Exp{\|Q\left(\nabla f_i(x^{k+1}) - \Gamma_{x^k}^{x^{k+1}}\nabla f_i(x^k)\right) - \nabla f(x^{k+1}) + \Gamma_{x^k}^{x^{k+1}}\nabla f(x^k)\|^2}\\
				& + (1 - p)\Exp{\norm{g^k - \nabla f(x^k)}^2}\\
				& 	= \frac{(1 - p)\omega}{n^2} \sum_{i=1}^n \Exp{\norm{\nabla f_i(x^{k+1}) - \Gamma_{x^k}^{x^{k+1}}\nabla f_i(x^k)}^2} + (1 - p)\Exp{\norm{g^k - \nabla f(x^k)}^2}.
			\end{align*}
			Using $L$-g-smoothness of $f_i$ together with the tower property~\eqref{eq:tower-property}, we obtain
			\begin{align}\label{eq:marina-first-recursion}
				\Exp{\norm{g^{k+1} - \nabla f(x^{k+1})}^2} &= \frac{(1 - p)\omega L^2}{n} \Exp{\|\Expmap_{x^{k}}^{-1}(x^{k+1}) \|^2}\\
				\nonumber& + (1 - p)\Exp{\norm{g^k - \nabla f(x^k)}^2}.
			\end{align}
			
			Next, we introduce a new notation: $\Phi^k = f(x^k) - f^{*} + \frac{\eta}{2p} \norm{g^k - \nabla f(x^k)}^2.$ Using \eqref{eq:l-g-smooth-inequality} and \eqref{eq:marina-first-recursion}, we establish the following inequality:
			\begin{align}\label{eq:marina_lyapunov_recursion}
				\nonumber\Exp{\Phi^{k+1}} & \leq \Exp{f(x^k) - f^{*}} - \frac{\eta}{2}\Exp{ \norm{\nabla f(x^k)}^2} \\
				\nonumber& - \left( \frac{1}{2\eta} - \frac{L}{2} \right) \Exp{\norm{\Expmap_{x^{k}}^{-1}(x^{k+1})}^2} + \frac{\eta}{2} \Exp{\norm{g^k - \nabla f(x^k)}^2}\\
				\nonumber& + \frac{\eta}{2p}\Exp{ \frac{\left( 1 - p\right)\omega L^2}{n} \|\Expmap_{x^{k}}^{-1}(x^{k+1}) \|^2 + (1 - p)\norm{g^k - \nabla f(x^k)}^2}\\
				\nonumber& = \Exp{\Phi^k} - \frac{\eta}{2} \Exp{\|\nabla f(x^k)\|^2}\\
				\nonumber& + \left( \frac{\eta(1 - p)\omega L^2}{2pn} - \frac{1}{2\eta} + \frac{L}{2} \right) \Exp{ \norm{\Expmap_{x^{k}}^{-1}(x^{k+1})}^2}\\
				& \leq \Exp{\Phi^k} - \frac{\eta}{2} \Exp{\norm{\nabla f(x^k)}^2},
			\end{align}
			where in the last inequality, we use $\frac{\eta(1 - p)\omega L^2}{2pn} - \frac{1}{2\eta} + \frac{L}{2} \leq 0$ following from the choice of the stepsize and Lemma~\ref{lemma:square_iequality}. Summing up inequalities~\eqref{eq:marina_lyapunov_recursion} for $k = 0, 1, \ldots, K - 1$ and rearranging the terms, we derive
			\[
			\frac{1}{K} \sum_{k=0}^{K-1} \Exp{\|\nabla f(x^k)\|^2} \leq \frac{2}{\eta K} \sum_{k=0}^{K-1} \left(\Exp{\Phi^k} - \Exp{\Phi^{k+1}}\right) = \frac{2\left(\Exp{\Phi^0} - \Exp{\Phi_K}\right)}{\eta K} \leq \frac{2\delta^0}{\eta K},
			\]
			since $g^0 = \nabla f(x^0)$ and $\Phi^{k+1}\geq 0,$ $\forall k = 0,1,\ldots,K-1.$ Finally, using the tower property~\eqref{eq:tower-property} and the definition of $\widehat{x}^K,$ we obtain~\eqref{eq:rmarina_convergence_guarantees}.
		\end{proof}
		\noindent\textbf{Corollary~\ref{cor:rmarina_noncvx}.}	\textit{With the assumptions outlined in Theorem~\ref{thm:marina_noncvx}, if we set $\eta =  \frac{1}{L \left(1 + \sqrt{\frac{1-p}{p}\frac{\omega}{n} }\right)}$, in order to obtain $\mathbb{E}\left[\|\nabla f\left(\widehat{x}^{K}\right)\|^2\right] \leq \varepsilon^2,$ we find that the communication complexity of Riemannian MARINA method (Algorithm~\ref{alg:R-MARINA}) is given by
			$$
			K = \mathcal{O} \left( \frac{\delta^0 L}{\varepsilon^2} \left(1 + \sqrt{\frac{(1-p) \omega}{pn}}\right) \right).
			$$}
		\begin{proof}[Proof of Corollary~\ref{cor:rmarina_noncvx}]
			The iteration complexity result follows directly from~\eqref{eq:rmarina_convergence_guarantees}.
		\end{proof}
		\noindent\textbf{Corollary~\ref{cor:rmarina_noncvx_total}.} 
		With the assumptions outlined in Theorem~\ref{thm:marina_noncvx}, if we set $\eta =  \frac{1}{L \left(1 + \sqrt{\frac{1-p}{p}\frac{\omega}{n} }\right)}$, $p = \frac{\rho_{\mathcal{Q}}}{d}$ in order to obtain 	$\mathbb{E}\left[\|\nabla f\left(\widehat{x}^{K}\right)\|^2\right] \leq \varepsilon^2,$ we find that the communication complexity $\mathcal{C}$ of Riemannian MARINA method (Algorithm~\ref{alg:R-MARINA}) is given by
		$$
		\mathcal{C} = \mathcal{O}\left(\frac{\delta_0 L}{\varepsilon^2}\left(1+\sqrt{\frac{\omega}{n}\left(\frac{d}{\rho_{\mathcal{Q}}}-1\right)}\right)\right)
		$$
		where $\rho_{Q}$ is the expected density of the quantization (see Def.~\ref{def:quantization}), and the expected total communication cost per worker is $\mathcal{O}\left(d+\rho_{\mathcal{Q}} K\right)$.
		\begin{proof}[Proof of Corollary~\ref{cor:rmarina_noncvx_total}] In order to obtain the first result, one needs to replace $p$ with $\frac{\rho_{Q}}{d}$ in the communication complexity expression of Corollary~\ref{cor:rmarina_noncvx}. Further, to retrieve the expected total communication cost per worker, we assume that the communication cost is proportional to the number of non-zero components sent:
			$$d + K\left(pd + (1-p) \rho_Q\right) = \mathcal{O} \left( d + \frac{\delta^0 L}{\varepsilon^2} \left(1 + \sqrt{\frac{(1-p) \omega}{pn}}\right) \left(pd + \left(1-p\right) \rho_Q\right) \right) = \mathcal{O}\left(d+\rho_{\mathcal{Q}} K\right).$$
		\end{proof}
		\section{R-MARINA: P\L-Condition}
		\begin{theorem}\label{thm:rmarina-pl}
			Let Assumptions~\ref{ass:l-g-smoothness}, \ref{ass:pl-condition} and \ref{ass:lower} be satisfied on $\mathcal{M}$ and
			\[
			\gamma \leq \min \left\{ \frac{1}{L \left( 1 + \sqrt{\frac{2(1-p)\omega}{pn}} \right)}, \frac{p}{2\mu} \right\}.
			\]
			Then after $K$ iterations of \algname{R-MARINA} we have
			\begin{equation}\label{eq:rmarina-pl-recursion}
				\mathbb{E} \left[ f(x^K) - f(x^*) \right] \leq (1 - \gamma\mu)^K \delta^0,
			\end{equation}
			where $\delta^0 = f(x^0) - f(x^*).$
		\end{theorem}
		\begin{proof}[Proof of Theorem~\ref{thm:rmarina-pl}]
			From~\eqref{eq:l-g-smooth-inequality} and P\L-condition (Assumption~\ref{ass:pl-condition}), we have
			\[
			\mathbb{E}[f(x^{k+1}) - f(x^*)] \leq \mathbb{E}[f(x^k) - f(x^*)] - \frac{\eta}{2}\mathbb{E}\left[  \left\| \nabla f(x^k) \right\|^2 \right] - \left( \frac{1}{2\eta} -  \frac{L}{2} \right) \mathbb{E}\left[ \left\| \Expmap_{x^{k}}^{-1}(x^{k+1}) \right\|^2 \right]
			\]
			\[+ \frac{\eta}{2} \mathbb{E}\left[ \left\| g^k - \nabla f(x^k) \right\|^2 \right]\]
			\[
			\leq (1 - \eta\mu)\mathbb{E}[f(x^k) - f(x^*)] - \left( \frac{1}{2\eta} -  \frac{L}{2} \right)\mathbb{E}\left[ \left\| \Expmap_{x^{k}}^{-1}(x^{k+1}) \right\|^2 \right] + \frac{\eta}{2} \mathbb{E}\left[ \left\| g^k - \nabla f(x^k) \right\|^2 \right].
			\]
			From~\eqref{eq:marina-first-recursion}, we have
			\begin{equation*}
				\mathbb{E} \left[ \|g^{k+1} - \nabla f(x^{k+1})\|^2 \right] \leq \frac{(1 - p)\omega L^2}{n} \mathbb{E} \left[ \|\Expmap_{x^{k}}^{-1}(x^{k+1})\|^2 \right] + (1 - p) \mathbb{E} \left[ \|g^k - \nabla f(x^k)\|^2 \right].
			\end{equation*}
			
			We derive that the sequence $\Phi^k = f(x^k) - f(x^*)  + \frac{\eta}{p}\|g^k - \nabla f(x^k)\|^2$ satisfies
			\begin{align*}
				\mathbb{E}[\Phi^{k+1}] &\leq \mathbb{E} \left[ (1 - \eta\mu)(f(x^k) - f(x^*)) - \left( \frac{1}{2\eta} - \frac{L}{2} \right) \|\Expmap_{x^{k}}^{-1}(x^{k+1})\|^2 + \frac{\eta}{2} \|g^k - \nabla f(x^k)\|^2 \right]\\
				&+ \frac{\eta}{p} \mathbb{E} \left[ \frac{(1 - p)\omega L^2}{n} \|\Expmap_{x^{k}}^{-1}(x^{k+1})\|^2 + (1 - p) \|g^k - \nabla f(x^k)\|^2 \right]\\
				&= \mathbb{E} \left[ (1 - \eta\mu)(f(x^k) - f(x^*)) + \left( \frac{\eta}{2} + \frac{\eta}{p} (1 - p) \right) \|g^k - \nabla f(x^k)\|^2 \right]\\
				&+ \left( \frac{\eta(1 - p)\omega L^2}{pn} - \frac{1}{2\eta} + \frac{L}{2} \right) \mathbb{E} \left[ \|\Expmap_{x^{k}}^{-1}(x^{k+1})\|^2 \right]\\
				& \leq (1 - \eta\mu) \mathbb{E}[\Phi^k],
			\end{align*}
			where in the last inequality, we use $\frac{\eta(1-p)\omega L^2}{pn} - \frac{1}{2\eta} + \frac{L}{2} \leq 0$ and  $\frac{\eta}{2} + \frac{\eta}{p} (1 - p) \leq (1 - \eta\mu)\frac{\eta}{p}.$ Unrolling the recurrence and using  $g^0 = \nabla f(x^0),$ we obtain
			\begin{equation*}
				\mathbb{E}[f(x^K) - f(x^*)] \leq \mathbb{E}[\Phi^K] \leq (1 - \eta\mu)^K\Phi^0 = (1 - \eta\mu)^K(\mathbb{E}[f(x^0) - f(x^*)]).
			\end{equation*}
		\end{proof}
		\begin{corollary}\label{cor:rmarina-pl-iteration-cxty}
			After
			\[
			K = \mathcal{O}\left( \max \left\{ \frac{1}{p}, \frac{L}{\mu} \left( 1 + \sqrt{\frac{(1-p)\omega}{pn}} \right) \right\} \log \frac{\delta^0}{\varepsilon} \right),
			\]
			iterations \algname{R-MARINA} produces such a point $x^K$ that $\mathbb{E}[f(x^K) - f(x^*)] \leq \varepsilon$.
		\end{corollary}
		\begin{proof}[Proof of Corollary~\ref{cor:rmarina-pl-iteration-cxty}]
			The result follows from~\eqref{eq:rmarina-pl-recursion}.
		\end{proof}
		\begin{corollary}\label{cor:rmarina-pl-total-cxty}
			The expected total communication cost per worker equals
			\[
			\mathcal{C} = \mathcal{O} \left( d + \max \left\{ \frac{1}{p}, \frac{L}{\mu} \left( 1 + \sqrt{\frac{(1-p)\omega}{pn}} \right) \right\} (pd + (1 - p)\rho_Q) \log \frac{\delta^0}{\varepsilon} \right),
			\]
			where $\rho_Q$ is the expected density of the quantization (see Def.~\ref{def:quantization}).
		\end{corollary}
		\begin{proof}[Proof of Corollary~\ref{cor:rmarina-pl-total-cxty}]
			To retrieve the expected total communication cost per worker, we assume that the communication cost is proportional to the number of non-zero components sent:
			$$d + K\left(pd + (1-p) \rho_Q\right) = \mathcal{O} \left( d + \max \left\{ \frac{1}{p}, \frac{L}{\mu} \left( 1 + \sqrt{\frac{(1-p)\omega}{pn}} \right) \right\} (pd + (1 - p)\rho_Q) \log \frac{\delta^0}{\varepsilon} \right).$$
		\end{proof}
		\begin{corollary}\label{cor:rmarina-pl-optimal-p}
			Let the assumptions of Theorem~\ref{thm:rmarina-pl} hold and $p = \frac{\rho_Q}{d}.$ If
			\[
			\eta \leq \min \left\{ \frac{1}{L \left( 1 + \sqrt{\frac{2(1-p)\omega}{pn}} \right)}, \frac{p}{2\mu} \right\},
			\]
			then \algname{R-MARINA} requires
			\[
			K = \mathcal{O} \left( \max \left\{ \frac{d}{\rho_Q}, \frac{L}{\mu} \left( 1 + \sqrt{\frac{\omega}{n} \left( \frac{d}{\rho_Q} - 1 \right) }\right) \right\} \log \frac{\delta^0}{\varepsilon} \right)
			\]
			iterations/communication rounds to achieve $\mathbb{E}[f(x^K) - f(x^*)] \leq \varepsilon$, and the expected total communication cost per worker is
			\[
			\mathcal{O} \left( d + \max \left\{ d, \frac{L}{\mu} \left( \rho_Q + \sqrt{\frac{\omega\rho_Q}{n} (d - \rho_Q)} \right) \right\} \log \frac{\delta^0}{\varepsilon} \right).
			\]
		\end{corollary}
		\begin{proof}[Proof of Corollary~\ref{cor:rmarina-pl-optimal-p}]
			The choice of $ p = \frac{\rho_Q}{d} $ implies
			\[
			\frac{1 - p}{p} = \frac{d}{\rho_Q} - 1,
			\]
			\[
			pd + (1 - p)\rho_Q \leq \rho_Q + \left( 1 - \frac{\rho_Q}{d} \right) \cdot \rho_Q \leq 2\rho_Q.
			\]
			Plugging these relations in into the results of Theorem~\ref{thm:rmarina-pl} and Corollaries~\ref{cor:rmarina-pl-iteration-cxty}~and~\ref{cor:rmarina-pl-total-cxty}, we get that if
			\[
			\eta \leq \min \left\{ \frac{1}{L \left( 1 + \sqrt{\frac{2\omega}{n} \left( \frac{d}{\rho_Q} - 1 \right)} \right)}, \frac{p}{2\mu} \right\},
			\]
			then \algname{R-MARINA} requires
			\[
			K = \mathcal{O} \left( \max \left\{ \frac{1}{p}, \frac{L}{\mu}\left( 1 + \sqrt{\frac{(1-p)\omega}{pn}}\right) \right\} \log \frac{\delta^0}{\varepsilon} \right)
			\]
			\[
			= \mathcal{O} \left( \max \left\{ \frac{d}{\rho_Q}, \frac{L}{\mu}\left( 1 + \sqrt{\frac{\omega}{n} \left( \frac{d}{\rho_Q} - 1 \right)}\right) \right\} \log \frac{\delta^0}{\varepsilon} \right)
			\]
			iterations in order to achieve $\mathbb{E}[f(x^K) - f(x^*)] \leq \varepsilon$, and the expected total communication cost per worker is
			\[d + K(pd + (1 - p)\rho_Q)=
			\]
			\[
			\mathcal{O} \left( d + \max \left\{ \frac{1}{p}, \frac{L}{\mu}\left(  1 + \sqrt{\frac{(1 - p)\omega}{pn}}\right) \right\} (pd + (1 - p)\rho_Q) \log \frac{\delta^0}{\varepsilon} \right)
			\]
			\[
			= \mathcal{O} \left( d + \max \left\{ d, \frac{L}{\mu} \left( \rho_Q + \sqrt{\frac{\omega\rho_Q}{n} (d - \rho_Q)} \right) \right\} \log \frac{\delta^0}{\varepsilon} \right)
			\]
			under an assumption that the communication cost is proportional to the number of non-zero components of transmitted vectors from workers to the server.
		\end{proof}
		\section{Auxiliary Results}\label{sec:supplementary}
		The following statement is Lemma~5 from \cite{richtarik2021ef21}.
		\begin{lemma}\label{lemma:square_iequality}
			Let $a,b>0$ be some constants. If $0\leq\eta\leq\frac{1}{\sqrt{a}+b},$ then $a\eta^2+b\eta \leq 1.$ In particular, $\eta\leq \min\left\lbrace \frac{1}{\sqrt{a}},\frac{1}{b} \right\rbrace.$ The bound is tight up to a factor of $2$ since $\frac{1}{\sqrt{a}+b}\leq \min\left\lbrace \frac{1}{\sqrt{a}},\frac{1}{b}\right\rbrace\leq \frac{2}{\sqrt{a}+b}.$
		\end{lemma}
		For a random vector $X\in\mathbb{R}^d$ and any deterministic vector $x\in\mathbb{R}^d,$ the variance can be decomposed as 
		\begin{equation}\label{eq:variance-decomposition}
			\Exp{\norm{X-\Exp{X}}^2} = \Exp{\norm{X-x}^2} - \norm{\Exp{X} - x}^2.
		\end{equation}
		
		For random vectors~$X,Y\in\mathbb{R}^d,$ we have that
		\begin{equation}\label{eq:tower-property}
			\Exp{X} = \Exp{\Exp{X|Y}}
		\end{equation}
		under an assumption that all expectations in the expression above are well-defined. 
		
		The next auxiliary statement is a modified version of Lemma~2 from~\citep{li2021page}.
		\begin{lemma}\label{lem:identity_for_dot_prod}
			Let $x \in \mathcal{M},$ $g\in T_x\mathcal{M}.$ Then for any $M \geq 0$ we have the identity
			\begin{eqnarray*}
				\langle \nabla f(x), -\eta g \rangle + \frac{M\eta^2}{2} \|g\|^2 &=& -\frac{\eta}{2} \|\nabla f(x)\|^2 - \left(\frac{1}{2\eta} - \frac{M}{2}\right) \|-\eta g\|^2\\
				&+& \frac{\eta}{2} \|g - \nabla f(x)\|^2.
			\end{eqnarray*}
		\end{lemma}
		\begin{proof}
			Indeed,
			\begin{align*}
				&\left\langle \nabla f(x), -\eta g \right\rangle + \frac{M\eta^2}{2} \|g\|^2  \\
				& =\left\langle \nabla f(x) - g, -\eta g \right\rangle + \left\langle g, -\eta g \right\rangle + \frac{M\eta^2}{2} \|g\|^2\\
				& = \eta \left\langle \nabla f(x) - g, -g \right\rangle - \left( \frac{1}{2\eta} - \frac{M}{2} \right) \eta^2\|g\|^2\\
				& = \eta \left\langle \nabla f(x) - g, \nabla f(x) - g - \nabla f(x) \right\rangle - \left( \frac{1}{2\eta} - \frac{M}{2} \right) \|\eta g\|^2\\
				& = \eta \| \nabla f(x) - g \|^2 - \eta \left\langle \nabla f(x) - g, \nabla f(x) \right\rangle - \left( \frac{1}{2\eta} - \frac{M}{2} \right) \|-\eta g\|^2\\
				& = \eta \| \nabla f(x) - g \|^2 - \left( \frac{1}{2\eta} - \frac{M}{2} \right) \|\eta g\|^2\\
				& - \frac{\eta}{2} \left( \|\nabla f(x) - g\|^2 + \|\nabla f(x)\|^2 - \|g\|^2 \right)\\
				& = \eta \| \nabla f(x) - g \|^2 - \left( \frac{1}{2\eta} - \frac{M}{2} \right) \|\eta g\|^2\\
				& - \frac{\eta}{2} \left( \| \nabla f(x) - g \|^2 + \| \nabla f(x) \|^2 - \|g\|^2 \right)\\
				& = -\frac{\eta}{2} \| \nabla f(x) \|^2 - \left( \frac{1}{2\eta} - \frac{M}{2} \right) \|\eta g\|^2 + \frac{\eta}{2} \| g - \nabla f(x) \|^2.\\
			\end{align*}
		\end{proof}
	\end{document}